\documentclass{article}
        
\usepackage{arxiv}

\usepackage[utf8]{inputenc} % allow utf-8 input
\usepackage{hyperref}       % hyperlinks
\usepackage{url}            % simple URL typesetting
\usepackage{xcolor}         % colors
\usepackage{booktabs}       % professional-quality tables
\usepackage{url}
\usepackage{amsmath}
\usepackage{amsthm}
\usepackage{mathtools}
\usepackage{wrapfig}

\DeclareMathOperator*{\argmin}{arg\,min}
\DeclarePairedDelimiter\floor{\lfloor}{\rfloor}
\newcommand{\D}{\mathcal{D}}

\newcommand{\G}{\mathcal{G}}
\newcommand{\E}{\mathbb{E}}
\newcommand{\R}{\mathbb{R}}
\newcommand{\CX}{\mathcal{X}}

\newcommand{\CF}{\mathcal{F}}

\newcommand{\CP}{\mathcal{P}}
\newcommand{\CM}{\mathcal{M}}
\newcommand{\CN}{\mathcal{N}}
\newcommand{\GNN}{\mathcal{G}_{NN}}
\newcommand{\DNN}{\mathcal{D}_{NN}}
\newcommand{\GNNS}{\mathcal{G}_{NN}^{\Sigma}}
\newcommand{\DNNS}{\mathcal{D}_{NN}^{\Sigma}}

\usepackage{amsfonts}
\usepackage{amssymb}
\usepackage{xcolor}
\usepackage{graphicx}
\usepackage{float}
\usepackage{enumerate}
\usepackage{enumitem}
\usepackage{multirow}

\newcommand{\norm}[1]{\left\lVert#1\right\rVert}
\newcommand{\abs}[1]{\left |#1\right |}
\usepackage{mathtools}

\usepackage{algorithmic, algorithm} %algpseudocode

\newtheorem*{rep@theorem}{\rep@title}
\newcommand{\newreptheorem}[2]{%
\newenvironment{rep#1}[1]{%
 \def\rep@title{#2 \ref{##1}}%
 \begin{rep@theorem}}%
 {\end{rep@theorem}}}
\makeatother

\newtheorem{theorem}{Theorem}

\newtheorem{assumption}{Assumption}
\newtheorem{definition}{Definition}
\newtheorem{lemma}{Lemma}
\newtheorem{proposition}{Proposition}

\newtheorem{remark}{Remark}
\newtheorem{example}{Example}
\newreptheorem{theorem}{Theorem}
\newreptheorem{lemma}{Lemma}

\usepackage[export]{adjustbox}
\usepackage{commath}
\usepackage{hyperref}
\usepackage{cleveref}
\usepackage{todonotes}

\allowdisplaybreaks
\hypersetup{
	colorlinks,
	linkcolor=black,
	anchorcolor=black,
	citecolor=black
}
\newcommand*\diff{\mathop{}\!\mathrm{d}}

\graphicspath{ {./figure/} }

\title{Statistical Guarantees of Group-Invariant GANs}

\author{  Ziyu Chen\\
  Department of Mathematics and Statistics\\
  University of Massachusetts Amherst\\
  Amherst, MA 01003,  USA \\
  \texttt{ziyuchen@umass.edu} \\
  \And
    Markos A. Katsoulakis\\
    Department of Mathematics and Statistics\\
  University of Massachusetts Amherst\\
  Amherst, MA 01003,  USA \\
  \texttt{markos@umass.edu} \\
\And
    Luc Rey-Bellet\\
    Department of Mathematics and Statistics\\
  University of Massachusetts Amherst\\
  Amherst, MA 01003,  USA \\
  \texttt{luc@umass.edu} 
  \And
    Wei Zhu\\
    Department of Mathematics and Statistics\\
  University of Massachusetts Amherst\\
  Amherst, MA 01003,  USA \\
  \texttt{weizhu@umass.edu} 
}

\begin{document}
\date{}

\maketitle

\begin{abstract}
This work presents the first statistical performance guarantees for group-invariant generative models. Many real data, such as images and molecules, are invariant to certain group symmetries, which can be taken advantage of to learn more efficiently as we rigorously demonstrate in this work. Here we specifically study generative adversarial networks (GANs), and quantify the gains when incorporating symmetries into the model. Group-invariant GANs are a type of GANs in which the generators and discriminators are hardwired with group symmetries. Empirical studies have shown that these networks are capable of learning group-invariant distributions with significantly improved data efficiency. In this study, we aim to rigorously quantify this improvement by analyzing the reduction in sample complexity and in the discriminator approximation error for group-invariant GANs. Our findings indicate that when learning group-invariant distributions, the number of samples required for group-invariant GANs decreases proportionally by a factor of the group size and the discriminator approximation error has a reduced lower bound. Importantly, the overall error reduction cannot be achieved merely through data augmentation on the training data. Numerical results substantiate our theory and highlight the stark contrast between learning with group-invariant GANs and using data augmentation. This work also sheds light on the study of other generative models with group symmetries, such as score-based generative models.
\end{abstract}


\section{Introduction}
The machine learning community has shown a growing interest in generative models, which aim to understand the underlying distribution of data and generate new samples from it. Among the various generative models, Generative Adversarial Networks (GANs) \cite{goodfellow2014generative} have garnered significant attention due to their ability to learn and sample from complex data distributions. Empirical evidence shows that GANs achieve remarkable performance in diverse applications such as image synthesis and text generation \cite{reed2016generative, yu2018generative, zhu2017unpaired}. Such success, however, hinges upon the availability of \textit{abundant} training data.

Recent research has proposed to leverage the group-invariant structure of underlying distributions to improve the data efficiency of GANs \cite{dey2020group, birrell2022structure}.  Additionally, related work has explored other generative models: \cite{kohler2020equivariant} proposes equivariant flows, while \cite{lu2024structure} and \cite{hoogeboom2022equivariant} study equivariant diffusion models. Such approaches are motivated by the prevalence of \textit{group symmetry} observed in various real-world distributions. For instance, in the case of medical images captured without orientation alignment, the distribution should be rotation-invariant, i.e., an image and its rotated copy are \textit{equiprobable}. Furthermore, many physical, physicochemical, and biochemical systems possess intrinsic symmetries or equivariance structures, making structured probabilistic modeling crucial \cite{pretti2020symmetry,ohlsson2020symmetry,jacobs2021symmetries,nicoli2020asymptotically,li2018neural,noid2013perspective,pak2018advances}. Empirical evidence from \cite{kohler2020equivariant,dey2020group, birrell2022structure} shows that models designed to respect group symmetry can effectively learn a group-invariant distribution even with limited data. The key idea is to introduce structures into generative models, i.e., moving away from
\textit{generic} models to exploit essential structures in physicochemical models. This shift
is expected to create \textit{physics-informed} generative models that are more sample-efficient and reduce computational load and complexity. However, a clear theoretical understanding of these phenomena remains to be established.

In this paper, we provide statistical 
performance guarantees %estimations
that explain why group-invariant generative models, in particular, group-invariant GANs \cite{dey2020group,birrell2022structure}, can effectively learn group-invariant distributions with significantly fewer training data. Specifically, let $\Sigma$ be a finite group, $\mu$ be a $\Sigma$-invariant target distribution supported on some compact domain $\CX\subset\R^d$, and $\rho$ be an easy-to-sample (noise) source distribution. Let $S^\Sigma[(g_{n,m}^*)_\sharp\rho]$ be the $\Sigma$-invariant generated distribution   learned based on $n$ i.i.d. training samples from the target $\mu$ and $m$ random draws from the noise source $\rho$; see Eq.~\eqref{eq:minimizer2} for the exact definition. We show that, if $m$ is sufficiently large and the network sizes are sufficiently large, then

\begin{equation}\label{eq:estimate_1}
\E\left[\mathcal{W}_1(S^\Sigma[(g_{n,m}^*)_\sharp\rho],\mu)\right] \, \text{is controlled by} \, \left(\frac{1}{\abs{\Sigma}n}\right)^{1/d},
\end{equation}
where $\mathcal{W}_1(P, Q)$ is the Wasserstein-1 distance between two distributions $P$ and $Q$; see \cref{eq:Wasserstein}. This reduction by a factor of $|\Sigma|$ in the number of training samples needed to learn a $\Sigma$-invariant target $\mu$  partially explains the  \textit{enhanced} data efficiency and generalization guarantees of GANs compared to existing generalization errors, such as those reported in \cite{ huang2022error}, by leveraging the group structure. This reduction in training sample complexity can be interpreted as follows: the performance of a group-invariant GAN utilizing $n$ i.i.d. training samples is equivalent to a vanilla GAN with $\abs{\Sigma}n$ i.i.d. training samples. This is crucial especially when the data available is scarce.
If we further assume that  $\mu$ is supported on a smooth $d^*$-dimensional submanifold of $\mathbb{R}^d$, then
\begin{equation}
\label{eq:estimate_2}
\E\left[\mathcal{W}_1(S^\Sigma[(g_{n,m}^*)_\sharp\rho],\mu)\right] \, \text{is controlled by} \, \left(\frac{1}{\abs{\Sigma}n}\right)^{1/d^*}.
\end{equation}

This estimation suggests that the improvement in the error bound does not suffer from the \textit{curse of dimensionality}, as it depends only on the \textit{intrinsic dimension} $d^*$ of the target's support, which could be much smaller than the ambient dimension $d$. More importantly, we demonstrate that group-invariant GANs function effectively as though the input training data is augmented by the group actions in an i.i.d. manner (in contrast, augmented data are conditionally independent but not i.i.d.) without an increase in the number of parameters and with a reduced discriminator approximation error lower bound compared to vanilla GANs, with fixed network architecture (see \Cref{prop:nonequivariance}). For a visual illustration of the significant difference between data augmentation and learning with group-invariant GANs, refer to the numerical experiments in \Cref{sec:numerical} (\cref{fig:student-t_2d}, \cref{fig:student-t_12d} and \cref{fig:w1distance}). We remark that we only present the dominating terms in \eqref{eq:estimate_1} and \eqref{eq:estimate_2} assuming $m$ is sufficiently large and the network sizes are sufficiently large, and their precise statement can be found in \cref{theorem:main} and \cref{theorem:lowdimensional}.

To the best of our knowledge, our work presents the first step towards theoretically understanding the impact of harnessing group symmetry within group-invariant generative models. While our primary focus lies on the performance guarantees for group-invariant GANs, we hope the analysis developed herein can be generalized in future works to study performance guarantees for other group-invariant generative models. See, for instance, \cite{chen2024equivariant} discussed at the end of \Cref{sec:relatedwork} and \Cref{remark:discriminator_lowerbound}.

This paper is organized as follows. In \Cref{sec:relatedwork}, we review some related work.  \Cref{sec:background} provides the background and goals of our study. Theoretical results when the target distribution lies in the Euclidean space or a low-dimensional submanifold are presented in \Cref{sec:euclidean} and \Cref{sec:manifold}, respectively. A numerical example is provided in \Cref{sec:numerical}. 
 We conclude our work and discuss future directions in \Cref{sec:conclusion}.  Some of the proofs are deferred to \Cref{proof:euclidean} and \Cref{proof:lowdimensional}.

\section{Related work}\label{sec:relatedwork}

A burgeoning body of recent research on group-equivariant neural networks \cite{cohen2016group,cohen2019general, weiler2019general} has demonstrated remarkable empirical success achieved through leveraging group symmetries in various supervised machine learning tasks. These accomplishments have, in turn, fostered the development of group-invariant/equivariant generative models, enabling data-efficient unsupervised learning of group-invariant distributions. Notable endeavors in this area include equivariant normalizing flows \cite{bilovs2021scalable,boyda2021sampling,garcia2021n,kohler2020equivariant}, group-equivariant GANs \cite{dey2020group}, structure-preserving GANs \cite{birrell2022structure} that rely on structure-customized divergences \cite{birrell2020f}, and discrete and continuous equivariant diffusion models \cite{hoogeboom2022equivariant, lu2024structure}. Despite the intuitive understanding, the extent to which and the underlying reasons why group symmetry can enhance the data efficiency of those group-invariant generative models remain largely unknown.

On the other hand, recent research \cite{chen2020distribution, huang2022error, liang2021well} has made significant progress in the sample-complexity analysis of \textit{vanilla} GANs (i.e., GANs without group-symmetric generators and discriminators) in distribution learning. These works quantify the disparity, often measured using probability divergences, between the distribution learned by GANs with finite training samples and the underlying data distribution. Compared to these works, our analysis follows a similar oracle inequality and bounds each term of the inequality separately. However, the novelty of our work lies in: 1) proposing a version of group-invariant generator and discriminator architectures; 2) treating each error term with incorporated group-invariant network architectures and utilizing the fact that the target distribution is group-invariant. While recent work \cite{chen2023sample,tahmasebi2023sample} has explored the enhanced sample complexity of divergence estimation under group symmetry, it remains unclear how these findings can be extended to the statistical estimation of group-invariant GANs aimed at learning group-symmetric probability distributions. Moreover, \cite{chen2023sample} in particular did not consider distributions with low-dimensional support, and their technical assumptions on the group actions are notably restrictive. In our work, we extend the analysis of GANs to include cases where distributions possess low-dimensional structure and we adopt a more flexible approach towards the assumptions on finite group actions compared to those in \cite{chen2023sample}.

While some other generative models such as diffusion models have gained traction recently, GANs as one-shot methods are still widely used in certain applications where computational efficiency and real-time generation are critical \cite{martinez2023ld,kang2023scaling,sauer2023stylegan}. Moreover, the mathematical formulation of GANs is essentially the minimization of variational divergences. As the objective of general generative modeling can be viewed as minimizing the discrepancy between the generated and the target distribution in terms of certain divergences, our work can shed light on the study of other generative models by considering the performance guarantees of invariant models evaluated by symmetry-adapted divergences. For example, in recent work \cite{chen2024equivariant}, we provide performance guarantees for score-based generative models with symmetries, inspired by this work.

\section{Background and problem setup}\label{sec:background}
\subsection{Integral Probability Metrics and GANs}
Let $\mathcal{X}$ be a measurable space, and $\CP(\CX)$ be the set of probability measures on $\CX$. Given $\Gamma\subset \mathcal{M}_b(\CX)$, where $\mathcal{M}_b(\CX)$ is the space of bounded measurable functions on $\CX$, the $\Gamma$-Integral Probability Metric ($\Gamma$-IPM) \cite{muller1997integral, sriperumbudur2012empirical} between $\nu\in\CP(\CX)$ and $\mu\in\CP(\CX)$ is defined as
\begin{equation}
    \label{eq:IPM}
    d_\Gamma(\nu, \mu) \coloneqq \sup_{\gamma\in\Gamma}\left\{\E_\nu[\gamma] - \E_\mu[\gamma]\right\}.
\end{equation}
For example, if $\Gamma = \text{Lip}_H(\CX)$, the space of $H$-Lipschitz functions
on $\CX$, then by the Kantorovich-Rubinstein duality, we have
    \begin{equation}
        \label{eq:Wasserstein}
        \mathcal{W}_1(\nu, \mu) =H^{-1}d_\Gamma(\nu, \mu).
    \end{equation}
    For simplicity, we omit the factor $H^{-1}$ in the following presentation.
    When $\Gamma$ is a class of neural networks, $d_\Gamma(\nu,\mu)$ defined by \eqref{eq:IPM} is called the “neural network distance” \cite{arora2017generalization}.
    
    GANs, originally proposed by Goodfellow et al. \cite{goodfellow2014generative}, learn to approximate a target data distribution $\mu$ by generating samples from a noise source distribution $\rho$ (typically chosen as uniform or Gaussian) through a two-player game. More specifically, GANs aim to solve the following minmax problem,
\begin{equation}
    \min_{g\in\G} \max_{f\in\D}\E_{x\sim\mu}[f(x)]-\E_{z\sim\rho}[f(g(z))],
\end{equation}
where $\G$ is a class of functions called \textit{generators} and $\D$ is a class of functions called \textit{discriminators}. This minmax problem can be equivalently formulated as
\begin{equation*}
    \min_{g\in\G} d_{\D}(\mu, g_\sharp\rho),
\end{equation*}
where $g_\sharp\rho \coloneqq \rho\circ g^{-1}$ is the push-forward measure of $\rho$ under $g$. For example, if $\D = \text{Lip}_1(\CX)$ is the class of 1-Lipschitz functions, this minmax problem is the formulation of Wasserstein GANs (W-GANs) \cite{arjovsky2017wasserstein}.

Suppose we have only finite samples $\{x_i\}_{i=1}^n$ from the target $\mu$ and $\{z_i\}_{i=1}^m$ from the source $\rho$, we define the optimal generator as the minimizer of the following optimization problem \cite{chen2020distribution,huang2022error},
\begin{equation}\label{eq:minimizer2vanilla}
g_{n,m}=\argmin_{g\in\G}d_{\D}(g_\sharp\widehat{\rho}_m,\widehat{\mu}_n),
\end{equation}
where $\widehat{\mu}_n = \frac{1}{n}\sum_{i=1}^n\delta_{x_{i}}$ is the empirical distribution of the target $\mu$ and $\widehat{\rho}_m = \frac{1}{m}\sum_{i=1}^m\delta_{z_{i}}$ is the empirical distribution of the source $\rho$.

\subsection{Group invariance and symmetrization operators}
A \textit{group} is a set $\Sigma$ equipped with a group product satisfying the properties of associativity, identity, and invertibility. Given a group $\Sigma$ and a set $\CX$, a map $\theta:\Sigma \times \CX\to\CX$ is called a \textit{group action on $\CX$} if $\theta_\sigma\coloneqq \theta(\sigma, \cdot): \CX\to\CX$ is a bijection on $\CX$ for any $\sigma\in\Sigma$, and $\theta_{\sigma_2}\circ \theta_{\sigma_1} = \theta_{\sigma_2\cdot \sigma_1}$, $\forall \sigma_1, \sigma_2\in\Sigma$. If the context is clear, we will abbreviate $\theta(\sigma, x)$ as $\sigma x$.

A function $\gamma:\CX\to \R$ is said to be \textit{$\Sigma$-invariant} if it is constant over group orbits of $\CX$, i.e., $\gamma\circ \theta_\sigma = \gamma$ for all $\sigma\in\Sigma$. Let $\Gamma
\subset \CM_b(\CX)$ be a set of bounded measurable functions $\gamma:\mathcal{X}\to\mathbb{R}$. We define its subset, $\Gamma_{\Sigma}\subset \Gamma$, of $\Sigma$-invariant functions as
\begin{equation}
\label{eq:invariant_function_space}
    \Gamma_{\Sigma} \coloneqq \{\gamma\in\Gamma:\gamma\circ \theta_\sigma=\gamma,\forall \sigma\in\Sigma\}.
\end{equation}
Likewise, a probability measure $\mu\in \CP(\CX)$ is said to be \textit{$\Sigma$-invariant} if $\mu = (\theta_\sigma)_\sharp \mu$ for all $\sigma\in \Sigma$. We define the set of all $\Sigma$-invariant distributions on $\CX$ as 
\begin{equation}
    \CP_\Sigma(\CX)\coloneqq \{\mu\in\CP(\CX): \mu ~\text{is}~ \Sigma\text{-invariant}\}.    
\end{equation}

Finally, following \cite{birrell2022structure, chen2023sample}, we define two \textit{symmetrization operators}, $S_\Sigma:\mathcal{M}_b(\CX)\to\mathcal{M}_b(\CX)$ and $S^\Sigma:\CP(\CX)\to\CP(\CX)$, on functions and probability measures, respectively, as
\begin{equation}
\label{eq:symmetrization_function}
    S_\Sigma[\gamma](x)  \coloneqq \int_\Sigma \gamma(\sigma x)\mu_\Sigma(d\sigma),~\forall \gamma\in\mathcal{M}_b(\CX),
\end{equation}
\begin{equation}
\label{eq:symmetrization_measure}
    \E_{S^\Sigma[\mu]} \gamma \coloneqq \E_\mu S_\Sigma[\gamma], ~\forall \mu\in\CP(\CX), \forall \gamma \in\mathcal{M}_b(\CX),
\end{equation}
where $\mu_\Sigma$ is the unique Haar probability measure on a compact Hausdorff topological group $\Sigma$ \cite{folland1999real}. One can easily verify that $S_\Sigma$ and $S^\Sigma$ are, respectively, \textit{projection operators} onto their corresponding invariant subsets $\Gamma_\Sigma\subset \Gamma$ and $\CP_\Sigma(\CX)\subset \CP(\CX)$ \cite{birrell2022structure}. One of the main results of \cite{birrell2022structure} that motivates the construction of group-invariant GANs is summarized in the following Lemma.
\begin{lemma}[paraphrased from \cite{birrell2022structure}]
\label{thm:sp-gan-main-result}
    If $S_\Sigma[\Gamma]\subset \Gamma$ and $\nu, \mu\in\CP(X)$, then
    \begin{equation*}
        d_\Gamma(S^\Sigma[\nu], S^\Sigma[\mu]) = d_{\Gamma_\Sigma}(\nu, \mu) = \sup_{\gamma\in \Gamma_{\Sigma}}\left\{\E_{\nu}[\gamma] - \E_\mu[\gamma]\right\}.
    \end{equation*}
    In particular, if $\nu, \mu\in\CP_\Sigma(\CX)$ are $\Sigma$-invariant, then $d_\Gamma(\nu, \mu) = d_{\Gamma_\Sigma}(\nu, \mu)$.
\end{lemma}
In other words, to ``tell the difference'' between two $\Sigma$-invariant distributions $\nu, \mu\in \CP_{\Sigma}(\CX)$, one only needs to optimize over all $\Sigma$-invariant discriminators $\gamma\in \Gamma_\Sigma$. \textit{Group-invariant GANs} are thus a type of GANs where the discriminators are all $\Sigma$-invariant and the generators are parameterized in such a way that the generated distributions are always $\Sigma$-invariant. We will further elaborate on the exact construction of group-invariant GANs in Section \ref{sec:architecture}.

\subsection{Assumptions on the group actions}
We assume $\CX\subset \R^d$ is a bounded subset of $\mathbb{R}^d$ equipped with the Euclidean metric $\|\cdot\|_2$. In addition, we make the following assumptions on the group $\Sigma$ and its action on $\CX$.
\begin{assumption}
\label{assumption:1}
    The group $\Sigma$ and its action on $\CX\subset \R^d$ satisfy\\
    1. $\Sigma$ is finite, i.e., $\abs{\Sigma}<\infty$;\\
    2. $\Sigma$ acts linearly on $\CX$; that is, $\sigma(ax_1+bx_2) = a\sigma x_1 + b\sigma x_2,\,\forall x_1,x_2\in\CX, \sigma\in\Sigma, a,b\in\mathbb{R}$;\\
    3. the $\Sigma$-actions on $\CX$ are isometric; that is,
    \begin{equation*}
        \norm{\sigma x_1-\sigma x_2}_2 = \norm{x_1-x_2}_2,\,\forall x_1,x_2\in\CX, \sigma\in\Sigma.
    \end{equation*}
\end{assumption}
\begin{remark}
    In terms of the third condition, one can instead assume that the $\Sigma$-actions on $\CX$ are 1-Lipschitz: that is, $\norm{\sigma x_1-\sigma x_2}_2 \leq \norm{x_1-x_2}_2,\,\forall x_1,x_2\in\CX, \sigma\in\Sigma$. Weyl's unitary trick tells us that for every finite group action, there exists a basis in which the action is isometric, such that the third condition can always be satisfied with a simple change of basis.
\end{remark}
The $\Sigma$ action on $\CX$ induces a \textit{fundamental domain} $\CX_0\subset\CX$, a terminology which we adopt from \cite{chen2023sample}.
\begin{definition}[Fundamental domain]
    \label{def:fundamental_domain}
    A subset $\CX_0\subset \CX$ is called a \textit{fundamental domain} of $\CX$ under the $\Sigma$-action if for any $x\in\CX$, there exists a unique $x_0\in\CX_0$ such that $x = \sigma x_0$ for some $\sigma\in\Sigma$.
\end{definition}

We remark that the choice of the fundamental domain $\CX_0$ is generally not unique. Following \cite{chen2023sample}, we  abuse the notation $\CX = \Sigma\times \CX_0$ to denote $\CX_0$ is a fundamental domain of $\CX$ under the $\Sigma$-action. Given a specific choice  of the fundamental domain $\CX_0$, we define the projection $T_0:\mathcal{X}\to\mathcal{X}_0$ as
\begin{align}\label{def:quotientmap}
T_0(x) \coloneqq y\in\CX_0, ~\text{if}~y=\sigma x ~\text{for some}~\sigma\in\Sigma.
\end{align}
In other words, $T_0$ maps any $x\in\CX$ to its unique orbit representative in $\CX_0$. Given a $\Sigma$-invariant distribution $\mu\in\CP_\Sigma(\CX)$ on $\CX$, the map $T_0$ induces a distribution $\mu_{\mathcal{X}_0}\in \CP(\CX_0)$ on the fundamental domain $\CX_0$ defined by
\begin{equation}
\label{eq:fundamental_domain_measure}
\mu_{\CX_0} = (T_0)_\sharp\mu.
\end{equation}

We next introduce an important concept,  the \textit{covering number}, which appears in many proofs throughout our paper.

\begin{definition}[Covering number]
\label{def:covering_number}
Let $(\CX,\tau)$ be a metric space. A subset $S\subset \CX$ is called an $\epsilon$-cover of $\mathcal{X}$ if for any $x\in \mathcal{X}$ there is an $s\in S$ such that $\tau(s,x)\leq\epsilon$. Define the $\epsilon$-covering number of $\mathcal{X}$ as
\[
\mathcal{N}(\mathcal{X},\epsilon,\tau):=\min\left\{\abs{S}:S \text{ is an } \epsilon\text{-cover of } \mathcal{X}\right\}.
\]
When $\tau(x, y) = \|x-y\|_2$ is the Euclidean metric in $\R^d$, we abbreviate $\mathcal{N}(\mathcal{X},\epsilon,\tau)$ as $\mathcal{N}(\mathcal{X},\epsilon)$.
\end{definition}

With Definition~\ref{def:fundamental_domain} and Definition~\ref{def:covering_number}, we specify below our second technical assumption on the group action over $\CX$, which is motivated by \cite{sokolic2017generalization, chen2023sample}.
\begin{assumption}\label{assumption:new}
There exists some constant $\epsilon_{\Sigma}$ and a set-valued function $A_0(\epsilon)$, such that for any $\epsilon\in(0,\epsilon_{\Sigma})$, we have $A_0(\epsilon)\subset\CX_0$, and\\
1. $\norm{\sigma x-\sigma'x'}_2>2\epsilon$, $\forall x,x'\in\CX_0\backslash A_0(\epsilon), \sigma\neq\sigma'\in\Sigma$;\\
2. $\norm{\sigma x-\sigma x'}_2 = \norm{x-x'}_2$, $\forall x,x'\in\CX_0\backslash A_0(\epsilon), \sigma\in\Sigma$.\\
Moreover, for some $0\leq r\leq d$, we have
\begin{equation}
\limsup_{\epsilon\to0^+}\frac{\CN(A_0(\epsilon),\epsilon)}{\CN(\CX_0,\epsilon)\epsilon^r}<\infty.
\end{equation}
\end{assumption}

\begin{remark}
Assumption~\ref{assumption:new} is notably less restrictive compared to the sufficient condition of Theorem 3 in \cite{sokolic2017generalization}, which can be viewed as a special case of Assumption~\ref{assumption:new} when $A(\epsilon)=\emptyset$ and $r=0$. Such relaxation is crucial to accommodate symmetries such as mirror reflections and circular rotations in the following examples.
\end{remark}
\begin{example}[Mirror reflection in $\mathbb{R}^2$]
$\CX = [-1,1]\times[0,1]$. The group actions are generated by $\sigma(x,y) = (-x,y)$. In this case, we have $r=1$ in Assumption~\ref{assumption:new}. See the left subfigure in Figure~\ref{fig:assumption}.
\end{example}
\begin{example}[Rotations in $\mathbb{R}^2$]
$\CX = (\rho\cos\theta,\rho\sin\theta),~~\rho\in[0,1],~~\theta\in[0,2\pi)$ and $\CX_0 = (\rho\cos\theta,\rho\sin\theta),~~\rho\in[0,1],~~\theta\in[0,\pi/2)$. The group actions are generated by the $\pi/2$-rotation with respect to the origin. In this case, we also have $r=1$ in Assumption~\ref{assumption:new}. See the right subfigure in Figure~\ref{fig:assumption}.
\end{example}
\begin{remark}
The introduction of $A_0$ implies that group actions are \textit{effective}: if $x_0\in\mathcal{X}_0$, then there is an infinitesimal lower bound $\epsilon$ between $\sigma x_0$ and $\mathcal{X}_0$ for any $\sigma\neq id$. For example, in Example 1, if $x_0\in\mathcal{X}_0$ is arbitrarily close to the border of $\mathcal{X}_0$, i.e., $\{0\}\times[0,1]$, then $\sigma_1 x_0$ will be arbitrarily close to $\mathcal{X}_0$, where $\sigma_1$ is the reflection action; similarly for Example 2, when $x_0\in\mathcal{X}_0$ is close to $\theta=0$ and $\sigma_1$ is a counterclockwise rotation by $\pi/2$. Also note that the assumptions in \cite{sokolic2017generalization, chen2023sample} both fail to satisfy in the two cases in Figure~\ref{fig:assumption}. 
\end{remark}

\begin{figure}[t]
    \centering
\centerline{\includegraphics[width=.5\linewidth]{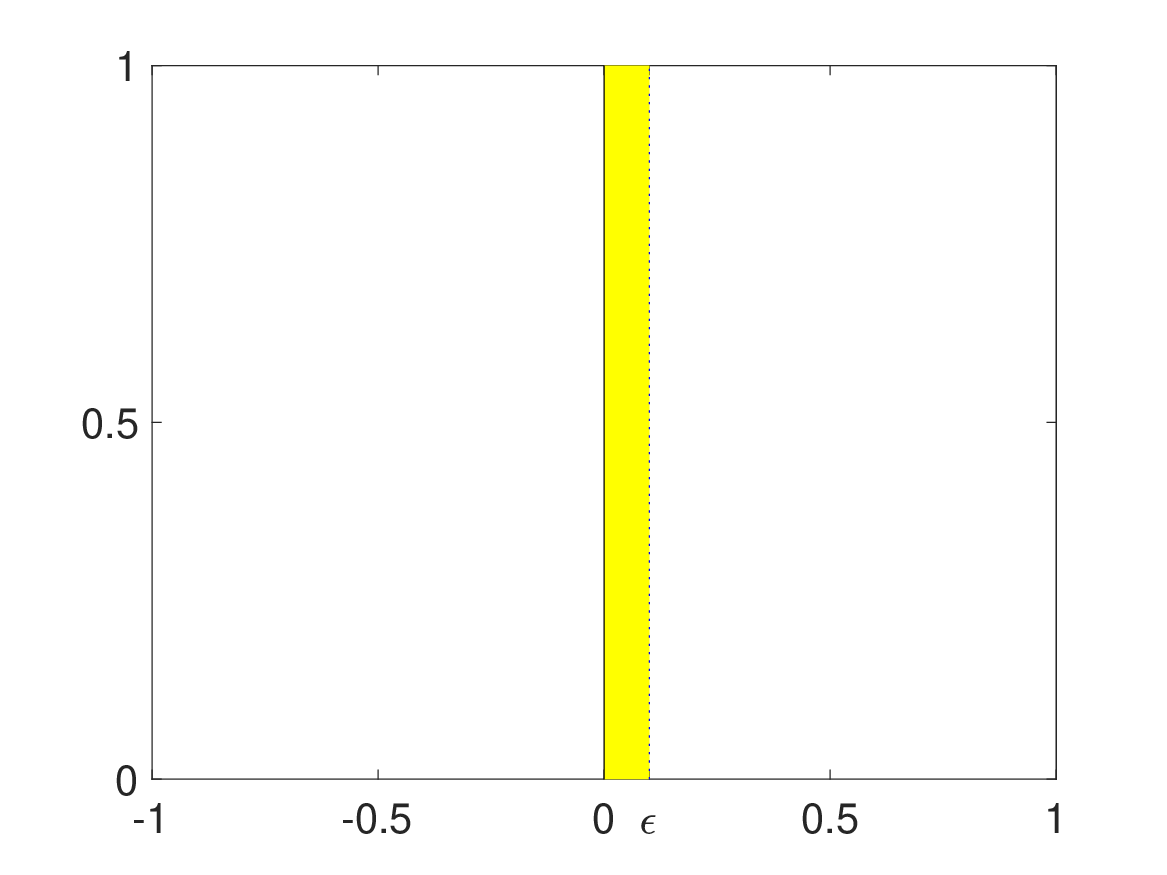}
\includegraphics[width=.5\linewidth]{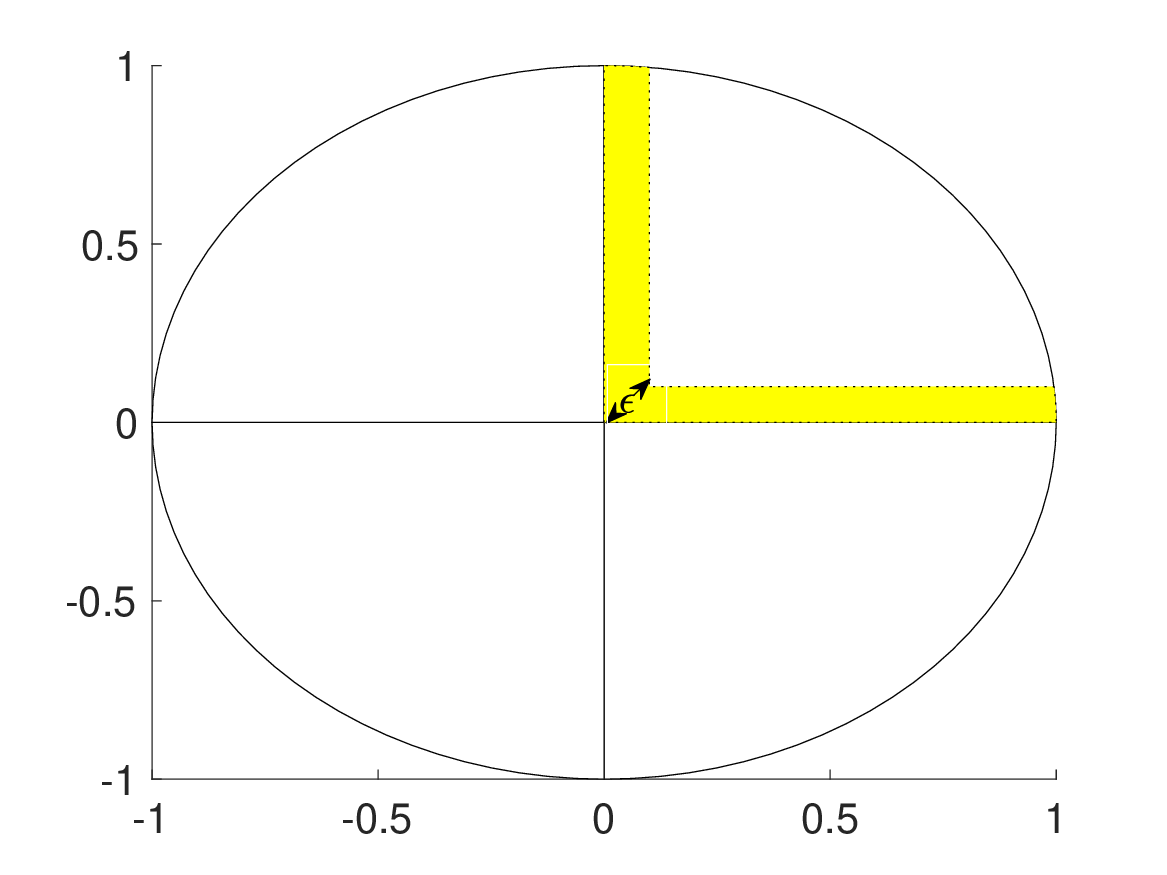}}
    \caption{Examples of Assumption~\ref{assumption:new}. Left: Mirror reflection in $\mathbb{R}^2$. $\CX_0 = [0,1]\times[0,1]$. The group actions are generated by $\sigma(x,y) = (-x,y)$. Right: Circular rotations within the unit disk in $\mathbb{R}^2$. The group actions are generated by the $\pi/2$-rotation with respect to the origin. $\CX_0 = (\rho\cos\theta,\rho\sin\theta),~~\rho\in[0,1],~~\theta\in[0,\pi/2)$. $A_0(\epsilon)$ are filled with yellow color in both subfigures.}
    \label{fig:assumption}
\end{figure}

\subsection{Network architecture for group-invariant GANs}
\label{sec:architecture}
In GANs, both the generator class and the discriminator class are parameterized by neural networks. In this paper, we consider fully-connected feed-forward ReLU networks as the architecture for both generators and discriminators.

A fully-connected feed-forward ReLU neural network, $\phi:\mathbb{R}^{d_0}\to\mathbb{R}^{d_{L+1}}$, can be formulated as 
\begin{equation}
    \phi(x) = W_L\cdot \text{ReLU}(W_{L-1}\cdots\text{ReLU}(W_0x+b_0)\cdots+b_{L-1})+b_L,
\end{equation}
where $W_i\in\mathbb{R}^{d_{i+1}\times d_i}$ and $b_i\in\mathbb{R}^{d_{i+1}}$ are network weights, $i=0,1,\dots,L$, and the activation function is given by $\text{ReLU}(x) := \max\{0,x\}$, $x\in\mathbb{R}$. We call the numbers $W=\max\{d_1,\dots,d_L\}$ and $L$ the width and the depth of the neural network $\phi$, respectively. When the input and output dimensions are clear from contexts, we denote by $\mathcal{NN}(W,L,N,K)$ a ReLU \textit{network architecture} that
consists of fully-connected feed-forward ReLU neural networks with width at most $W$,
depth at most $L$, the number of weights at most $N$ and the magnitude of entries of $W_i$ and $b_i$ no more than $K$. We also write $\mathcal{NN}(W,L)$ or $\mathcal{NN}(W,L,N)$ when either $N$ or $K$ is not specified.

Let $\DNN$ and $\GNN$, respectively,  denote a class of discriminators and a class of generators realized by fully-connected ReLU networks, which are generally not $\Sigma$-invariant. We explain next how to construct $\Sigma$-invariant discriminators and generators from given $\DNN$ and $\GNN$.

Since $\Sigma$ acts linearly on $\CX$ (Assumption~\ref{assumption:1}), there is a matrix representation $W_{\sigma}\in \R^{d\times d}$ for any $\sigma\in\Sigma$ such that $\sigma x = W_{\sigma} x, \forall x\in \R^d$. We can formulate $\Sigma$-invariant discriminators, denoted by $\DNNS$, using the function symmetrization operator $S_\Sigma$ introduced in \eqref{eq:symmetrization_function}.

\begin{definition}[$\Sigma$-invariant discriminators]\label{def:sigmarelu}
    For any class of discriminators with ReLU network architecture $\DNN = \mathcal{NN}(W,L,N,K)$, its corresponding class of $\Sigma$-invariant discriminators $\DNNS$ is defined as
    \begin{equation}
    \DNNS:=\left\{\frac{1}{\abs{\Sigma}}\sum_{i=1}^{\abs{\Sigma}}\phi(\sigma_i^{-1}x):\phi(x)\in\DNN\right\}.
    \end{equation}
    That is, $\DNNS$ consists of neural networks of the form
    \begin{align}\label{eq:invariant-discriminator}
   \phi(x) = \frac{1}{\abs{\Sigma}}\sum_{i=1}^{\abs{\Sigma} }W_L\cdot \text{ReLU}(W_{L-1}\cdots\text{ReLU}(W_0(W_{\sigma_i}x)+b_0)\cdots+b_{L-1})+b_L,
    \end{align}
where $\{W_i, b_i:i=0,\dots,L\}$ is the same set of weights as that of $\DNN$.
\end{definition}
By Definition \ref{def:sigmarelu}, it is straightforward to verify that $S_{\Sigma}\left[\DNNS\right] = \DNNS$.

\begin{remark}
    \cref{def:sigmarelu} requires $\abs{\Sigma}$ passes through the network; however, in numerous practical scenarios of generative modeling, particularly in medical imaging, the foremost limitation is typically the scarcity of sufficient data, rather than computational resources. Consequently, while \cref{def:sigmarelu} increases computational demands, it represents a necessary trade-off to effectively address this significant challenge in sample-constrained applications.
\end{remark}
Similarly, the $\Sigma$-invariant generators can be defined using the distribution symmetrization operator $S^\Sigma$ introduced in \eqref{eq:symmetrization_measure}.
\begin{definition}[$\Sigma$-invariant generators]\label{def:sigmagenerator}
     For any class of generators with ReLU network architecture $\GNN = \mathcal{NN}(W,L,N,K)$, its corresponding class of $\Sigma$-invariant generators $\GNNS$ is defined as 
    \begin{equation*}
    g_\Sigma\in\GNNS \iff \exists g(x)\in\GNN ~~s.t.~~ g_\Sigma(\rho) = S^\Sigma[g_\sharp\rho],~~\forall~ \text{noise source distribution}~\rho.
    \end{equation*}
\end{definition}
In other words, we add a $\Sigma$-symmetrization layer $S^\Sigma$ that draws a random $\sigma$ according to the Haar measure $\mu_\Sigma$ and transforms the output $g(x)$ of $\GNN$ to $\sigma g(x)$, as the output of $\GNNS$.

In this paper, we consider learning a $\Sigma$-invariant target data distribution $\mu\in\CP_\Sigma(\CX)$ using $\Sigma$-invariant GANs whose generators and discriminators are both $\Sigma$-invariant and the noise source probability measure $\rho$ is absolutely continuous on $\mathbb{R}$ (with respect to the Lebesgue measure). More specifically, in contrast to \eqref{eq:minimizer2vanilla}, we denote
\begin{equation}\label{eq:minimizer2}
g_{n,m}^*=\argmin_{g\in\GNN}d_{\DNNS}(S^\Sigma[g_\sharp\widehat{\rho}_m],\widehat{\mu}_n).
\end{equation}
Note that $\{S^\Sigma[g_\sharp\rho]:g\in\GNN\}$ is a set of $\Sigma$-invariant distributions in $\CX$ for any easy-to-sample source distribution $\rho$ on $\mathbb{R}$. 
\begin{remark}
    As mentioned in Remark 7 of \cite{huang2022error}, our results in the following sections also hold if $\rho$ is an absolutely continuous probability measure on $\mathbb{R}^k$, since any such $\rho$ can be projected to an absolutely continuous probability measure on $\mathbb{R}$ by a linear mapping.
\end{remark}
\begin{remark}
In \cite{chen2020distribution, huang2022error}, the generator $g_n^*=\argmin_{g\in\GNN}d_{\DNNS}(S^\Sigma[g_\sharp\rho],\widehat{\mu}_n)$ is also considered. We will not provide the detail for $g_n^*$, since in practice GANs are trained with finite samples from $\rho$. Moreover, the bound in our main result Theorem~\ref{theorem:main} directly applies to the case for $g_n^*$.
\end{remark}

\subsection{Notation}
We write $A\lesssim B$ if $A\leq CB$ for some constant $C>0$, and if both $A\lesssim B$ and $B\gtrsim A$ hold, then we write $A\simeq B$. We also write $A\lesssim_s B$ if the factor $C$ depends on some quantity $s$. Given a subset $\CX\subset \R^d$ of $\R^d$, its \textit{diameter} is denoted as $\text{diam}(\CX) \coloneqq \sup_{x, y\in \CX}\|x-y\|_2$. For two functions $f(n)$ and $g(n)$, we write $f(n) = o\left(g(n)\right)$ when $\lim_{n\to\infty} f(n)/g(n)= 0$.

\section{Learning distributions in Euclidean spaces}\label{sec:euclidean}
In this section, we quantify the error between the generated distribution $S^\Sigma[(g_{n,m}^*)_\sharp\rho]$ and the target data distribution $\mu\in\CP_\Sigma(\CX)$ in terms of the Wasserstein-1 metric; that is, $\Gamma=\text{Lip}_H(\CX)$ and we estimate $d_{\Gamma}(S^\Sigma[(g_{n,m}^*)_\sharp\rho],\mu)$. All of the proofs can be found in the appendix. Since the value of the $\Gamma$-IPM does not change if we replace $\Gamma$ by $\Gamma+c$ for any $c\in\mathbb{R}$, by Lemma 2 in \cite{chen2023sample}, it is equivalent to replace $\Gamma$ by $\Gamma=\{f: f\in\text{Lip}_H(\CX), \norm{f}_\infty\leq M\}$ if $\CX$ is bounded, where $M=H\cdot\text{diam}(\CX_0)$.

The proof of Theorem~\ref{theorem:main} hinges upon the following error decomposition based on the proof of Lemma 9 in \cite{huang2022error} and Lemma~\ref{thm:sp-gan-main-result}. 
Compared to \cite{huang2022error}, our analysis makes use of the group-invariant architectures of the generator and discriminator as well as the fact that the target distribution has group-invariant structure so that each error term will reveal the gain or no-extra-cost of group invariance.
\begin{lemma}[Error decomposition]\label{lemma:errordecompsition}
    Suppose $g_{n,m}^*$ is the optimal GAN generator from \eqref{eq:minimizer2}, then for any function class $\Gamma$ defined on $\CX\subset\mathbb{R}^d$, we have
    \begin{align*}
        &d_{\Gamma}(S^\Sigma[(g_{n,m}^*)_\sharp\rho],\mu)
        \leq \underbrace{2\sup_{f\in\Gamma_\Sigma}\inf_{f_\omega\in\DNNS}\norm{f-f_\omega}_\infty}_{\substack{\text{Invariant discriminator approximation error: $\Delta_1$}}}\\
        &\quad+ \underbrace{\inf_{g\in\GNN}d_{\DNNS}(S^\Sigma[g_\sharp\rho],\widehat{\mu}_n)}_{\substack{\text{Invariant generator approximation error: $\Delta_2$}}} + \underbrace{d_{\Gamma_\Sigma}(\widehat{\mu}_n,\mu)}_\text{Statistical error from target: $\Delta_3$}\\
        &\quad+ \underbrace{2d_{\DNNS\circ \GNN}(\rho,\widehat{\rho}_m)}_\text{Statistical error from source: $\Delta_4$}.
    \end{align*}
\end{lemma}
We provide lemmas that bound each of $\Delta_i$'s.

\begin{lemma}[$\Sigma$-invariant discriminator approximation error]\label{lemma:discriminator}
Let 
\[\Gamma=\{f: f\in\text{Lip}_H(\CX), \norm{f}_\infty\leq H\cdot\text{diam}(\CX_0)\}.\] 
For any $\epsilon\in(0,1)$, there exists a class of $\Sigma$-invariant discriminators \[\DNNS = S_\Sigma\left[\mathcal{NN}(W_1,L_1,N_1,K_1)\right]\]
with $L_1\lesssim \log(1/\epsilon)$, $N_1\lesssim \epsilon^{-d}\log(1/\epsilon)$ and $K_1$ depends on $H, \sup_{x\in\CX}\norm{x}_\infty$ and $\text{diam}(\CX_0)$, such that
\begin{equation*}
    \sup_{f\in\Gamma_\Sigma}\inf_{f_\omega\in\DNNS}\norm{f-f_\omega}_\infty\leq\epsilon.
\end{equation*}
\end{lemma}

In addition to \Cref{lemma:discriminator}, we take a closer look at each individual approximation to illustrate the importance of using $\Sigma$-invariant discriminators in the next proposition, whose proof can be found in \Cref{proof:Delta1}.
\begin{proposition}\label{prop:nonequivariance}
    Suppose $f$ is a $\Sigma$-invariant function and $f_\omega$ is not necessarily $\Sigma$-invariant. Assuming $\Sigma$ has unitary matrix representations, i.e., the Jacobian of each $\theta_\sigma$ is a unitary matrix, then we have
    \begin{align*}
        \norm{f-f_\omega}_\infty &\geq \sqrt{\frac{1}{\text{vol}(\CX)}\left(\norm{f-S_\Sigma[f_\omega]}_{L^2(\CX)}^2 + \norm{f_\omega-S_\Sigma[f_\omega]}_{L^2(\CX)}^2\right)}\\
        &\geq \sqrt{\frac{1}{\text{vol}(\CX)}} \norm{f_\omega-S_\Sigma[f_\omega]}_{L^2(\CX)}.
    \end{align*}
\end{proposition}
\begin{remark}[Deviation from invariance]\label{remark:discriminator_lowerbound}
    The right-hand side of \Cref{prop:nonequivariance} can be viewed as a measurement of the non-invariance of $f_\omega$ and can be large if $f_\omega$ is ``far from'' being $\Sigma$-invariant. This lower bound vanishes whenever $f_\omega$ is $\Sigma$-invariant, which suggests that we use $\Sigma$-invariant discriminators $\DNNS$ instead of $\DNN$ to achieve a smaller discriminator approximation error. See also \Cref{fig:w1distance} for numerical illustrations. This measurement for the deviation from invariance seems general in generative models with symmetries. For example,  \cite{chen2024equivariant} uses a similar notion to measure the non-equivariance of a vector field in the context of score-based generative models for score-matching approximations. 
\end{remark}

\begin{lemma}[$\Sigma$-invariant generator approximation error]\label{lemma:generator}
Suppose $W_2\geq7d+1, L_2\geq 2$. Let $\rho$ be an absolutely
continuous probability measure on $\mathbb{R}$. If $n\leq\frac{W_2-d-1}{2}\floor{\frac{W_2-d-1}{6d}}\floor{\frac{L_2}{2}}+2$, we have
$\inf_{g\in\GNN}d_{\DNNS}(S^\Sigma[g_\sharp\rho],\widehat{\mu}_n)=0$, where $\GNN=\mathcal{NN}(W_2,L_2)$.
\end{lemma}

\begin{remark}
    We also refer to Theorem 4.6 in \cite{birrell2022structure} for the necessity of using $\Sigma$-invariant generators: suppose the discriminator architecture is $\Sigma$-invariant. Since $d_{\DNNS}(g_\sharp\rho,\widehat{\mu}_n) = d_{\DNN}(S^\Sigma[g_\sharp\rho],S^\Sigma[\widehat{\mu}_n])$, the generator may only learn to generate distributions whose symmetrization is the target distribution, thus causing ``mode collapse''. That is, $d_{\DNNS}$ is a probability metric on $\CP_\Sigma(\CX)$ rather than $\CP(\CX)$. We also refer to qualitative numerical examples of this phenomenon in \cite{birrell2022structure}.
\end{remark}

\begin{lemma}[Statistical error from the target]\label{lemma:statistical-target}
    Let $\mathcal{X} = \Sigma\times\mathcal{X}_0$ be a subset of $\mathbb{R}^d$ satisfying the conditions in Assumptions~\ref{assumption:1} and~\ref{assumption:new}. Let $\Gamma=\{f: f\in\text{Lip}_H(\CX), \norm{f}_\infty\leq M\}$. Suppose $\mu$ is a $\Sigma$-invariant probability measure on $\mathcal{X}$. Then we have
    
\noindent(1) If $d\geq 2$, then for any $s>0$ and $n$ sufficiently large, we have
\begin{equation*}
\E[d_{\Gamma_\Sigma}(\widehat{\mu}_n,\mu)]\leq\quad C_{\CX,H,d,s}\left(\frac{1}{\abs{\Sigma}n}\right)^{\frac{1}{d+s}}+ o\left(\left(\frac{1}{n}\right)^{\frac{1}{d+s}}\right),
\end{equation*}
where the coefficient $C_{\CX,H,d,s}$ of the dominating term does not depend on $\Sigma$ or $n$;\\
(2) If $d=1$ and $\CX_0$ is an interval, then for $n$ sufficiently large, we have
\begin{equation*}
\E[d_{\Gamma_\Sigma}(\widehat{\mu}_n,\mu)]\leq  \frac{cH\cdot\text{\normalfont diam}(\mathcal{X}_0)}{\sqrt{n}} + o\left(\frac{1}{n}\right),
\end{equation*}
where $c>0$ is an absolute constant independent of $\mathcal{X}$ and $\mathcal{X}_0$.
\end{lemma}

\begin{lemma}[Statistical error from the source]\label{lemma:statistical-source}
    Suppose $L_1\lesssim\log n$, $N_1\lesssim n\log n$, $W_2^2L_2\lesssim n$, then we have $d_{\DNNS\circ \GNN}(\rho,\widehat{\rho}_m)\lesssim \sqrt{\frac{n^2\log^2 n\log m}{m}}$. In particular, if $m\gtrsim n^{2+2/d}\log^3n$, we have $\E[d_{\DNNS\circ \GNN}(\rho,\widehat{\rho}_m)]\lesssim n^{-1/d}$.
\end{lemma}
With those lemmas, we now state and prove the main theorem.
\begin{theorem}[Main Theorem]\label{theorem:main}
    Let $\CX=\Sigma\times\CX_0$ be a compact subset of $\mathbb{R}^d (d\geq 2)$ satisfying Assumptions~\ref{assumption:1} and~\ref{assumption:new}, and let $\Gamma=\text{Lip}_H(\CX)$. Suppose the target distribution $\mu$ is $\Sigma$-invariant on $\CX$ and the noise source distribution $\rho$ is absolutely continuous on $\mathbb{R}$. Then there exists $\Sigma$-invariant discriminator architecture $\DNNS=S_\Sigma[\DNN]$, where $\DNN=\mathcal{NN}(W_1,L_1,N_1)$ as defined in \eqref{eq:invariant-discriminator} with $N_1\lesssim n\log n$ and $L_1\lesssim \log n$, and $\Sigma$-invariant generator architecture $\GNNS$, where $\GNN = \mathcal{NN}(W_2,L_2)$, with $W_2^2L_2\lesssim n$, such that if $m\gtrsim n^{2+2/d}\log^3n$, we have
    \begin{equation}\label{eq:thm1formula}
        \E\left[d_{\Gamma}(S^\Sigma[(g_{n,m}^*)_\sharp\rho],\mu)\right] \leq  C_{\CX,H,d,s}\left(\frac{1}{\abs{\Sigma}n}\right)^{\frac{1}{d+s}}+ o\left(\left(\frac{1}{n}\right)^{\frac{1}{d+s}}\right),
    \end{equation}
    for any $s>0$, where the coefficient $C_{\CX,H,d,s}$ of the dominating term does not depend on $\Sigma$ or $n$.
    %$o\left(\left(\frac{1}{n}\right)^{\frac{1}{d+s}}\right)$ decays faster than $\left(\frac{1}{n}\right)^{\frac{1}{d+s}}$ in $n$ with the coefficient depends on $\Sigma$.
    See \eqref{eq:fullconstant} in the Appendix for details.
\end{theorem}

\begin{proof}
    We can choose $\epsilon\simeq n^{-1/d}$ in Lemma~\ref{lemma:discriminator} with $L_1\lesssim \log n^{1/d}\simeq\log n$ and $N_1\lesssim n\log n$, so that $\sup_{f\in\Gamma_\Sigma}\inf_{f_\omega\in\DNNS}\norm{f-f_\omega}_\infty\lesssim n^{-1/d}$. On the other hand, we can make $\inf_{g\in\GNN}d_{\DNNS}(S^\Sigma[g_\sharp\rho],\widehat{\mu}_n)=0$ with $W_2^2L_2\lesssim n$ by Lemma~\ref{lemma:generator}. Finally, $\E[d_{\Gamma_\Sigma}(\widehat{\mu}_n,\mu)]\lesssim_{s}(\abs{\Sigma}n)^{-1/(d+s)}$ and $\E[d_{\DNNS\circ \GNN}(\rho,\widehat{\rho}_m)]\lesssim n^{-1/d}$ by Lemma~\ref{lemma:statistical-target} and Lemma~\ref{lemma:statistical-source} respectively.
\end{proof}

\begin{remark}[Data efficiency]\label{rmk:vanilla}
    This improved training sample complexity can be interpreted as follows: For vanilla GANs without built-in group symmetry, the group size is $1$, representing the baseline sample complexity. However, for a group-invariant GAN with nontrivial group $\abs{\Sigma}>1$,  its performance when utilizing $n$ i.i.d. training samples is equivalent to a vanilla GAN with $\abs{\Sigma}n$ i.i.d. training samples. 
\end{remark}
\begin{remark}[Optimality of the rate]\label{rmk:optimalrate}
    The rate obtained in Theorem~\ref{theorem:main} is less sharp than $\lesssim n^{-1/d}\vee n^{-1/2}\log n$ provided in \cite{huang2022error} since in \cite{huang2022error} it is assumed that $\CX=\CX_0=[0,1]^d$ but here we do not assume the fundamental domain is connected. If we assume that $\CX_0$ is connected, then the rate can be improved to $\lesssim (\abs{\Sigma}n)^{-1/d}$ for $d\geq3$ and $\lesssim_s (\abs{\Sigma}n)^{-1/2}\log n$ for $d=2$ since \cref{lemma:coverbycover} in the proof can be replaced by a sharper covering number bound in \cite{kolmogorov1959e}.
\end{remark}

\section{Target distributions with low-dimensional structure}\label{sec:manifold}
While Theorem~\ref{theorem:main} explains the improved data efficiency of group-invariant GANs, the error bound increases as the ambient dimension $d$ increases. However, empirical evidence suggests that numerous probability distributions found in nature are concentrated on submanifolds of low \textit{intrinsic dimensions}. In light of this, we present an enhanced analysis of sample complexity for group-invariant GANs, specifically when learning distributions of low intrinsic dimension under group symmetry. Suppose the target distribution $\mu$ is supported on some compact $d^*$-dimensional smooth submanifold of $\mathbb{R}^d$. 
We first prove the following bound for the covering number of a compact smooth submanifold of $\mathbb{R}^d$, which reflects the intrinsic dimension of the submanifold.
\begin{lemma}\label{lemma:manifoldcovering}
    Let $\CM$ be a compact $d^*$-dimensional smooth submanifold of $\mathbb{R}^d$, then for small $\epsilon>0$, we have 
    \[
    \CN(\CM,\epsilon)\leq C_{\CM}(\frac{1}{\epsilon})^{d^*},
    \]
    where the metric in the covering number is the Euclidean metric on $\mathbb{R}^d$ and the constant $C_\CM$ depends on $\CM$.
\end{lemma}
\begin{proof}
    For any point $x\in\CM$, there exists a chart $(U_x,\varphi_x)$, where $U_x\subset\mathbb{R}^d$ is an open set containing $x$ and $\varphi_x:U_x\cap\CM\to \varphi_x(U_x\cap\CM)\subset\mathbb{R}^{d^*}$ is a diffeomorphism such that $\varphi_x(U_x\cap\CM)$ is an open set of $\mathbb{R}^{d^*}$. Without loss of generality, we can assume $\varphi_x(x)=0\in\mathbb{R}^{d^*}$. In addition, we can assume that $\varphi_x(U_x\cap\CM) = B_{d^*}(0,R_x)$ for some $R_x>0$, and $\norm{D\varphi_x^{-1}}_2\leq c_x$ on $B_{d^*}(0,R_x)$ for some constant $c_x$, where $D\varphi_x^{-1}$ is the Jacobian matrix of $\varphi_x^{-1}$. (This is possible since $\varphi_x(U_x)$ is open, we can pick some $r>0$ such that $B_{d^*}(0,r)\subset \varphi_x(U_x)$, so that $\overline{B_{d^*}(0,r/2)}$ is compact, and $\norm{D\varphi_x^{-1}}_2$ is bounded on $\overline{B_{d^*}(0,r/2)}$. Then we can set $U_x\cap\CM = \varphi_x^{-1}(B_{d^*}(0,r/4))$.) 
    
    Let $\epsilon_x\leq\frac{R_x}{3}$, then any open ball in $\mathbb{R}^{d^*}$ that intersects $B_{d^*}(0,\epsilon_x)$ must lie within $B_{d^*}(0,R_x)$. For any small $\epsilon\in(0,1)$, we know that $B_{d^*}(0,\epsilon_x)$ can be covered by $\frac{(\epsilon_x+\epsilon/2)^{d^*}}{(\epsilon/2)^{d^*}}$ balls with radius $\epsilon$ in $\mathbb{R}^{d^*}$ (c.f. Proposition 4.2.12 in \cite{vershynin2018high}). Moreover, since $\norm{D\varphi_x^{-1}}_2\leq c_x$, $\varphi_x^{-1}$ maps each of these $\epsilon$-balls to some open subset of $U_x$ whose diameter is no more than $c_x\epsilon$ and hence contained in some open ball of radius $2c_x\epsilon$ in $\mathbb{R}^d$. Therefore, $\varphi_x^{-1}(B_{d^*}(0,\epsilon_x))$ can be covered by $\frac{(\epsilon_x+\epsilon/2)^{d^*}}{(\epsilon/2)^{d^*}}$ balls with radius $2c_x\epsilon$ in $\mathbb{R}^{d}$. By a rescaling, $\varphi_x^{-1}(B_{d^*}(0,\epsilon_x))$ can be covered by $\frac{(\epsilon_x+\epsilon/4c_x)^{d^*}}{(\epsilon/4c_x)^{d^*}}$ balls with radius $\epsilon$ in $\mathbb{R}^{d}$. Also note that $\cup_{x\in\CM}[\varphi_x^{-1}(B_{d^*}(0,\epsilon_x))]$ forms an open cover of $\CM$, so there exist finitely many $x_1,\dots,x_n$ such that $\CM = \cup_{i=1}^n[\varphi_{x_i}^{-1}(B_{d^*}(0,\epsilon_{x_i}))]$, since $\CM$ is compact. Thus $\CM$ can be covered by $\sum_{i=1}^n \frac{(\epsilon_{x_i}+\epsilon/4c_{x_i})^{d^*}}{(\epsilon/4c_{x_i})^{d^*}}= \sum_{i=1}^n \left(\frac{4c_{x_i}\epsilon_{x_i}}{\epsilon}+1\right)^{d^*}\leq \sum_{i=1}^n \left(\frac{4c_{x_i}\epsilon_{x_i}+1}{\epsilon}\right)^{d^*}\leq\frac{C_\CM}{\epsilon^{d^*}}$ balls with radius $\epsilon$ in $\mathbb{R}^{d}$, where $C_{\CM} = n\cdot\max_{i}\{(4c_{x_i}\epsilon_{x_i}+1)^{d^*}\}$. That is, $\CN(\CM,\epsilon)\leq C_{\CM}(\frac{1}{\epsilon})^{d^*}$.
\end{proof}

\begin{theorem}\label{theorem:lowdimensional}
    Let $\CX=\Sigma\times\CX_0\subset\mathbb{R}^d$ and $\CX_0$ be a compact $d^*$-dimensional smooth submanifold of $\mathbb{R}^d (d^*\geq 2)$ satisfying Assumption~\ref{assumption:new} with some $0\leq r\leq d^*$ and $\Gamma=\text{Lip}_H(\CX)$. Suppose the target distribution $\mu$ is $\Sigma$-invariant on $\CX$ and the noise source distribution $\rho$ is absolutely continuous on $\mathbb{R}$. Then there exists $\Sigma$-invariant discriminator architecture $\DNNS=S_\Sigma[\DNN]$, where $\DNN=\mathcal{NN}(W_1,L_1,N_1)$ as defined in \eqref{eq:invariant-discriminator} with $N_1\lesssim n\log n$ and $L_1\lesssim \log n$, and $\Sigma$-invariant generator architecture $\GNNS$, where $\GNN = \mathcal{NN}(W_2,L_2)$, with $W_2^2L_2\lesssim n$, such that if $m\gtrsim n^{2+2/d^*}\log^3n$, we have
    \begin{equation*}
        \E\left[d_{\Gamma}(S^\Sigma[(g_{n,m}^*)_\sharp\rho],\mu)\right] \leq C_{\CX,H,d^*,s}\left(\frac{1}{\abs{\Sigma}n}\right)^{\frac{1}{d^*+s}}+ o\left(\left(\frac{1}{n}\right)^{\frac{1}{d^*+s}}\right),
    \end{equation*}
    for any $s>0$, where the coefficient $C_{\CX,H,d^*,s}$ of the dominating term does not depend on $\Sigma$ or $n$.
\end{theorem}

\begin{proof}
    Following the notation in the proof of Lemma~\ref{lemma:statistical-target}, we have by Lemma~\ref{lemma:coverbycover}, \[\log\mathcal{N}(\mathcal{F}_0,\epsilon,\,\norm{\cdot}_{\infty})\leq\mathcal{N}(\mathcal{X}_0,\frac{c_2\epsilon}{H})\log(\frac{c_1M}{\epsilon}),\]
    where $\mathcal{F}_0 = \{\gamma\in\text{Lip}_H(\mathcal{X}_0):\norm{\gamma}_\infty\leq M\}$. By Assumption~\ref{assumption:new} and Lemma~\ref{lemma:manifoldcovering}, we have
    \begin{align*}
    \mathcal{N}(\mathcal{X}_0,\frac{c_2\epsilon}{H})\log(\frac{c_1M}{\epsilon})
    &\leq \CN(\CX_0\backslash A_0(\frac{c_2\epsilon}{H}),\frac{c_2\epsilon}{H})\log(\frac{c_1M}{\epsilon}) + \CN(A_0(\frac{c_2\epsilon}{H}),\frac{c_2\epsilon}{H})\log(\frac{c_1M}{\epsilon})\\
    &\leq \frac{\mathcal{N}(\mathcal{X},\frac{c_2\epsilon}{H})}{\abs{\Sigma}}\log(\frac{c_1M}{\epsilon}) + \bar{c}\epsilon^r\CN(\CX,\frac{c_2\epsilon}{H})\log(\frac{c_1M}{\epsilon})\\
    &\leq \frac{C_{\CX}H^{d^*}}{\abs{\Sigma}c_2^{d^*}\epsilon^{d^*}}\log(\frac{c_1M}{\epsilon}) + \frac{\bar{c}C_{\CX}H^{d^*}}{c_2^{d^*}\epsilon^{d^*-r}}\log(\frac{c_1M}{\epsilon})
\end{align*}
when $\frac{c_2\epsilon^*}{H}\leq\epsilon_\Sigma$. Therefore we have the statistical error from the target $\Delta_3$ as
\[
\E[d_{\Gamma_\Sigma}(\widehat{\mu}_n,\mu)]\leq\quad C_{\CX,H,d^*,s}\left(\frac{1}{\abs{\Sigma}n}\right)^{\frac{1}{d^*+s}}+ o\left(\left(\frac{1}{n}\right)^{\frac{1}{d^*+s}}\right).
\]
The rest of the proof then follows that of Theorem~\ref{theorem:main}.
\end{proof}

\begin{remark}
    The statement of the theorem for 1-dimensional submanifolds can be found in the appendix.
\end{remark}

\begin{remark}
    The condition in \cref{theorem:lowdimensional} can be relaxed to be $C^1$-submanifolds, as the proof of \cref{lemma:manifoldcovering} only requires $C^1$ condition.
\end{remark}
\begin{remark}
    If the target distribution does not lie on a manifold, such as on fractals, we can introduce the Minkowski dimension, whose definition itself involves the covering number, so that we can derive the bound similarly as in Theorem~\ref{theorem:lowdimensional}.
\end{remark}

\section{Numerical experiments}\label{sec:numerical}
We provide a synthetic example to illustrate efficient learning by incorporating group symmetry into the generator and discriminator networks. Specifically, the target is a mixture of four 2-dimensional Gaussian distributions centered at $[\pm10,\pm10]$ with standard deviations of 5, resulting in a target distribution with $C_4$ symmetry.

We train three different GANs based on the vanilla GAN that has 3 hidden fully-connected layers with 64, 32, and 16 nodes in sequence with ReLU activations for both the generator and the discriminator. The discriminators and generators of $C_2$ and $C_4$ GANs are defined as in \Cref{def:sigmarelu} and \Cref{def:sigmagenerator} with $C_2$ and $C_4$ groups, respectively. We applied a gradient penalty as described in \cite{gulrajani2017improved} with a (soft) Lipschitz constant of 1 and penalty of 10 (see Theorem 31 in \cite{birrell2020f} for a theoretical justification). The noise source is selected to be a 10-dimensional Gaussian.

\begin{figure}[t!]
  \centering
  \includegraphics[width=.9\textwidth]{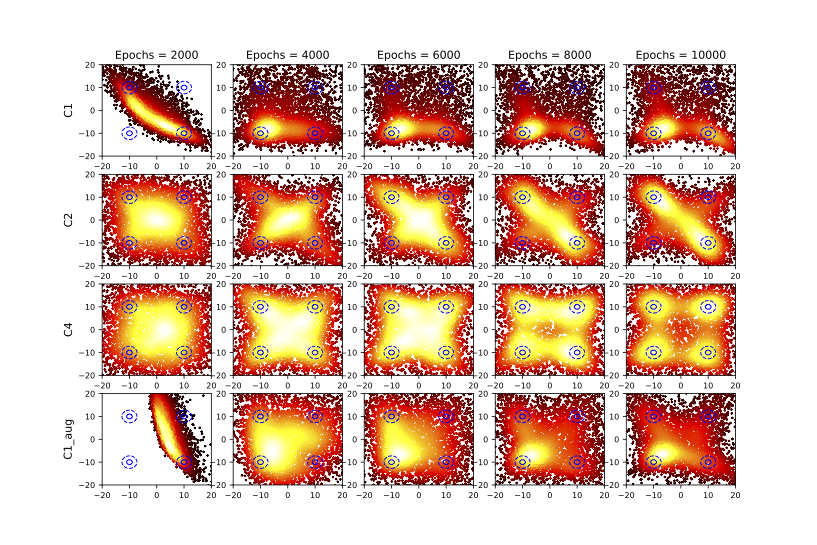}
  \caption{Heat map of 5000 GAN generated samples learned from a 2D Gaussian mixture. Top row: GAN with no symmetry; second row: GAN with $C_2$ 
 (partial) symmetry; third row: GAN with $C_4$ (full) symmetry; bottom row: GAN with no symmetry but with $C_4$ augmentation on the training data. Brighter parts refer to higher density.}
  \label{fig:student-t_2d}
\end{figure}

\cref{fig:student-t_2d} displays 5000 samples generated by the trained generator of each GAN, where we use 100 training samples, incorporating $C_1$ (no invariance, or vanilla), $C_2$, and $C_4$ symmetries. The GAN with $C_4$ symmetry (third row) achieves the best result, and the one with no invariance (top row) performs the worst. Additionally, the vanilla GAN without symmetry, even with full data augmentation, i.e., augmenting all the training data using group actions in $C_4$, thus we have 400 non-i.i.d. training samples (bottom row), performed only slightly better than the GAN without symmetry. This highlights the stark contrast between structured distribution learning with group-invariant GANs and using data augmentation.

Moreover, \cref{fig:student-t_12d} presents similar results with the same (hyper)parameters, except that the 2D Gaussian mixture is linearly embedded into a higher-dimensional ambient space $\mathbb{R}^{12}$. 

In \Cref{fig:w1distance}, we quantitatively evaluate the performance of methods by calculating the Wasserstein-1 distance between 10000 samples drawn from the generated distributions by the trained GANs for 10000 epochs and the target distribution of Gaussian mixture, respectively. We apply the linear program in \cite{flamary2021pot,flamary2024pot} for the calculation of the Wasserstein-1 distance, and we set the maximal number of iterations to $10^6$, and no cases reaching this limit were reported. For each training data size, we perform 20 independent runs for each method. We observe that the $C_4$ GAN consistently achieves the best performance in both examples. This is because the $C_4$ GAN not only has a smaller statistical error $\Delta_3$, but also provides a smaller invariant discriminator approximation error $\Delta_1$, which becomes the dominant error in \Cref{lemma:errordecompsition} when the number of training samples is large. (See \Cref{prop:nonequivariance} and \Cref{remark:discriminator_lowerbound} for the reduced $\Delta_1$ error by the invariant architecture.) We keep the same network architecture and hyperparameters in our numerical experiments to ensure consistency of performance across different methods.

\begin{figure}[t!]
  \centering
  \includegraphics[width=.9\textwidth]{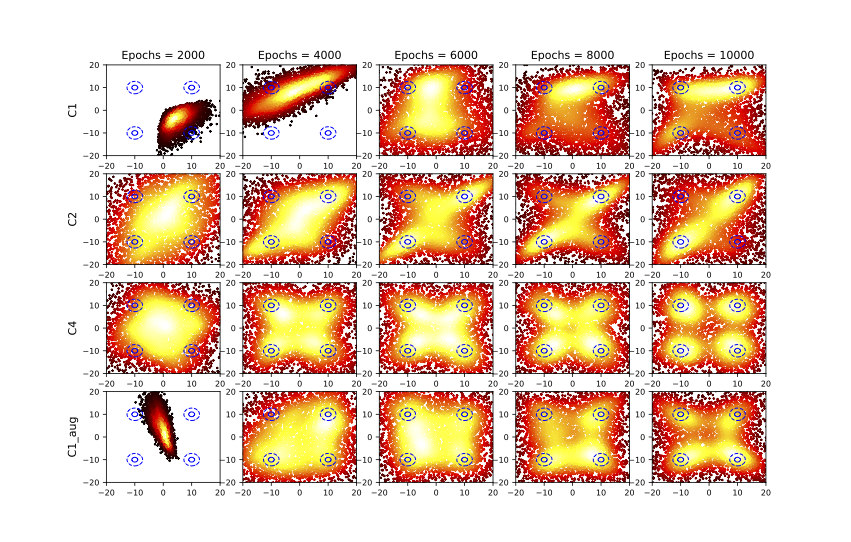}
  \caption{Heat map of 5000 GAN generated samples learned from a 2D Gaussian mixtur embedded in a higher-dimensional ambient space $\R^{12}$. The figure shows the 2D projection of the generated samples onto the intrinsic support plane ($d^*=2$) of the distribution. Top row: GAN with no symmetry; second row: GAN with $C_2$ 
 (partial) symmetry; third row: GAN with $C_4$ (full) symmetry; bottom row: GAN with no symmetry but with $C_4$ augmentation on the training data. Compared to \Cref{fig:student-t_2d}, this result qualitatively suggests that the convergence depends only on the intrinsic dimension, as discussed in \cref{theorem:lowdimensional}.}
  \label{fig:student-t_12d}
\end{figure}

\begin{figure}[h!]
  \centering
  
  \includegraphics[width=.48\textwidth]{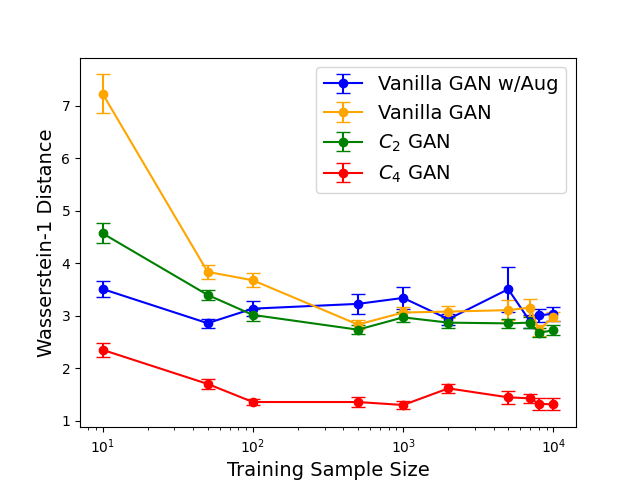}
    \includegraphics[width=.48\textwidth]{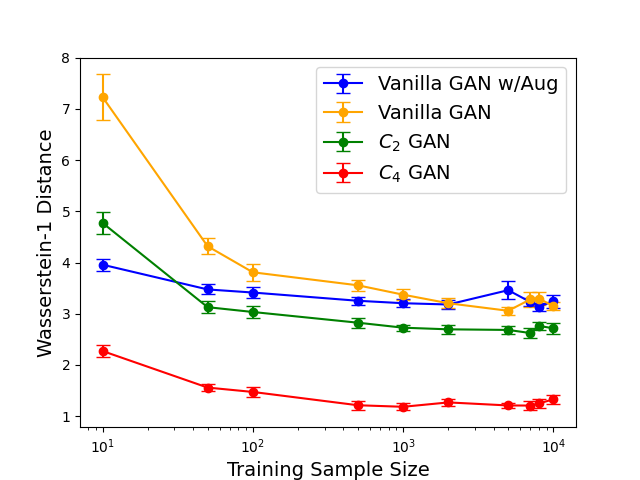}
  \caption{Wasserstein-1 distance between 10000 samples drawn from the generated and the target distribution with different GAN implementations, over 20 runs. Left: 2D example. Right: 12D example. $C_4$ GAN achieves the best performance in both cases.}  \label{fig:w1distance}
\end{figure}

\section{Conclusion and future work}\label{sec:conclusion}
In this work, we quantitatively established the improved generalization guarantees of group-invariant GANs. Our results indicate that when learning group-invariant target distributions, potentially supported on a manifold with small intrinsic dimension, the required number of samples for group-invariant GANs is reduced by a factor of the group size. In essence, group-invariant GANs function effectively as though the input data is augmented in the statistical sense, without an increase in the number of parameters, compared to vanilla GANs. This becomes crucial when we face data scarcity. On the other hand, however, the overall improved bound cannot simply be derived from data augmentation since invariant GANs have a reduced discriminator approximation error, and our numerical results further support this distinction. Our findings present several potential avenues for further exploration. Firstly, it is essential to explore how the dimension of the support of the noise source $\rho$ influences the performance of GANs. A recent experimental study \cite{zhu2024noise} suggests that the dimension of the source should not be significantly smaller than the intrinsic dimension of the target distribution's support. 
Secondly, exploring the case of unbounded support for the target distribution, particularly when heavy-tailed distributions are involved, holds promise for further investigation. To our knowledge, the theoretical study of how the dimension of the support and the distribution characteristics of the noise source $\rho$ influence the performance of GANs is still open.
Thirdly, it will also be worthwhile to study the statistical estimation and expressive power of group-invariant GANs constructed from equivariant convolutional neural networks (CNNs), based on the universality result of CNNs (cf. \cite{zhou2020universality}). Lastly, our work assumes the group is finite, which is crucial for defining our group-invariant generators and discriminators. Extending the analysis to compact Lie groups should also be explored. While the statistical error from the target $\Delta_3$ has been studied in \cite{chen2023sample,tahmasebi2023sample} for continuous groups, it remains an open problem for a novel design of invariant discriminator architecture and its universality for infinite groups.

\section{Proofs of results related to \cref{theorem:main}}\label{proof:euclidean}

\subsection{Proof of \cref{lemma:errordecompsition}}
\begin{proof}
Based on the definition of $d_{\Gamma}$ in Eq.~\eqref{eq:IPM}, we have
\begin{align*}
    d_{\Gamma}(S^\Sigma[(g_{n,m}^*)_\sharp\rho],\mu) &= d_{\Gamma_\Sigma}(S^\Sigma[(g_{n,m}^*)_\sharp\rho],\mu)\\
    & \leq d_{\Gamma_\Sigma}(S^\Sigma[(g_{n,m}^*)_\sharp\rho],\widehat{\mu}_n) + d_{\Gamma_\Sigma}(\widehat{\mu}_n,\mu)\\
    &\leq d_{\DNNS}(S^\Sigma[(g_{n,m}^*)_\sharp\rho],\widehat{\mu}_n) + 2\sup_{f\in\Gamma_\Sigma}\inf_{f_\omega\in\DNNS}\norm{f-f_\omega}_\infty + d_{\Gamma_\Sigma}(\widehat{\mu}_n,\mu)\\
    &\leq d_{\DNNS}(S^\Sigma[(g_{n,m}^*)_\sharp\rho],S^\Sigma[(g_{n,m}^*)_\sharp\widehat{\rho}_m]) + d_{\DNNS}(S^\Sigma[(g_{n,m}^*)_\sharp\widehat{\rho}_m],\widehat{\mu}_n)\\
    & \quad +2\sup_{f\in\Gamma_\Sigma}\inf_{f_\omega\in\DNNS}\norm{f-f_\omega}_\infty + d_{\Gamma_\Sigma}(\widehat{\mu}_n,\mu),
\end{align*}
where the first equality is due to Lemma~\ref{thm:sp-gan-main-result}, the first inequality is given by the triangle inequality, the second inequality is due to Lemma~24 in \cite{huang2022error}, and the last inequality is again by the triangle inequality. 

Therefore we have two different terms to bound: $d_{\DNNS}(S^\Sigma[(g_{n,m}^*)_\sharp\rho],S^\Sigma[(g_{n,m}^*)_\sharp\widehat{\rho}_m])$ and $d_{\DNNS}(S^\Sigma[(g_{n,m}^*)_\sharp\widehat{\rho}_m],\widehat{\mu}_n)$. For the first term, we have
\[
d_{\DNNS}(S^\Sigma[(g_{n,m}^*)_\sharp\rho],S^\Sigma[(g_{n,m}^*)_\sharp\widehat{\rho}_m]) = d_{\DNNS}((g_{n,m}^*)_\sharp\rho,(g_{n,m}^*)_\sharp\widehat{\rho}_m),
\]
due to Lemma~\ref{thm:sp-gan-main-result} and the fact that $S_\Sigma[\DNNS]=\DNNS$. Hence we have
\begin{align*}
    d_{\DNNS}((g_{n,m}^*)_\sharp\rho,(g_{n,m}^*)_\sharp\widehat{\rho}_m)
    &= \sup_{h\in \DNNS}\left\{ E_{(g_{n,m}^*)_\sharp\rho}[h] - E_{(g_{n,m}^*)_\sharp\widehat{\rho}_m}[h]\right\}\\
    & = \sup_{h\in \DNNS}\left\{ E_{\rho}[h\circ g_{n,m}^*] - E_{\widehat{\rho}_m}[h\circ g_{n,m}^*]\right\}\\
    & \leq d_{\DNNS\circ \GNN}(\rho,\widehat{\rho}_m)
\end{align*}

For the second term, by the definition in \eqref{eq:minimizer2}, we have for any $g\in\GNN$,
\begin{align*}
d_{\DNNS}(S^\Sigma[(g_{n,m}^*)_\sharp\widehat{\rho}_m],\widehat{\mu}_n) 
&\leq d_{\DNNS}(S^\Sigma[g_\sharp\widehat{\rho}_m],\widehat{\mu}_n)\\
&\leq d_{\DNNS}(S^\Sigma[g_\sharp\widehat{\rho}_m],S^\Sigma[g_\sharp\rho]) + d_{\DNNS}(S^\Sigma[g_\sharp\rho],\widehat{\mu}_n)\\
&\leq d_{\DNNS\circ \GNN}(\rho,\widehat{\rho}_m) + d_{\DNNS}(S^\Sigma[g_\sharp\rho],\widehat{\mu}_n).
\end{align*}
Taking the infimum over $g\in\GNN$, we have 
\[
d_{\DNNS}(S^\Sigma[(g_{n,m}^*)_\sharp\widehat{\rho}_m],\widehat{\mu}_n) \leq d_{\DNNS\circ \GNN}(\rho,\widehat{\rho}_m) + \inf_{g\in\GNN}d_{\DNNS}(S^\Sigma[g_\sharp\rho],\widehat{\mu}_n).
\]
Therefore,
\begin{align*}
    d_{\Gamma}(S^\Sigma[(g_{n,m}^*)_\sharp\rho],\mu)\leq & \inf_{g\in\GNN}d_{\DNNS}(S^\Sigma[g_\sharp\rho],\widehat{\mu}_n) + 2\sup_{f\in\Gamma_\Sigma}\inf_{f_\omega\in\DNNS}\norm{f-f_\omega}_\infty\\
    & + d_{\Gamma_\Sigma}(\widehat{\mu}_n,\mu) + 2d_{\DNNS\circ \GNN}(\rho,\widehat{\rho}_m).
\end{align*}
This completes the proof of \cref{lemma:errordecompsition}.
\end{proof}

\subsection{Bound for $\Delta_1$}\label{proof:Delta1}
We first cite the following network universal approximation result from \cite{yarotsky2017error}.
\begin{lemma}[Universal Approximation of Lipschitz functions, Theorem 1 in \cite{yarotsky2017error}]\label{thm:universal}
Let $\CX$ be a compact domain in $\mathbb{R}^d$ and $\sup_{x\in\CX}\norm{x}_\infty\leq F$. Given any $\epsilon\in(0,1)$, there exists a ReLU network architecture $\mathcal{NN}(W,L,N)$ such that, for any $f\in\text{Lip}_H(\CX)$ and $\norm{f}_\infty\leq M$, with appropriately chosen network weights, the network provides a function $\hat{f}$ such that $\norm{f-\hat{f}}_\infty\leq\epsilon$. Such a network has no more than $c_1(\log\frac{1}{\epsilon}+1)$ layers, and at most $c_2\epsilon^{-d}(\log\frac{1}{\epsilon}+1)$ neurons and weight parameters, where the constants $c_1$ and $c_2$ depend on $d,H,F$ and $M$.
\end{lemma}
\begin{remark}
    Though the original Theorem 1 in \cite{yang2022capacity} assumes $\CX$ is a cube in $\mathbb{R}^d$, it also holds when $\CX$ is some compact domain, since we can extend any $f\in\text{Lip}_H(\CX)$ to some cube contains $\CX$ while preserving the Lipschitz constant, by the Kirszbraun theorem \cite{kirszbraun1934zusammenziehende}.
\end{remark}

\begin{proof}[Proof of \cref{lemma:discriminator}]
By Lemma \ref{thm:universal}, for any $\epsilon\in(0,1)$, there exists a ReLU network architecture $\mathcal{NN}(W_1,L_1,N_1,K_1)$ such that
\begin{equation*}
    \sup_{f\in\Gamma_\Sigma}\inf_{\phi\in\mathcal{NN}(W_1,L_1,N_1,K_1)}\norm{f(x)-\phi(x)}_\infty\leq\epsilon,
\end{equation*}
where $W_1,L_1,N_1$ are given by \cref{thm:universal} and $K_1$ depends on $H$, $\sup_{x\in\CX}\norm{x}_\infty$ and $M$. 
Hence for any $f\in\Gamma_\Sigma$ and $\delta>0$, there exists $\phi\in\mathcal{NN}(W_1,L_1,N_1,K_1)$ such that
\begin{align*}
    \norm{f(x)-\phi(\sigma^{-1}x)}_\infty
    &=\norm{f(\sigma^{-1}x)-\phi(\sigma^{-1}x)}_\infty\\
    &=\norm{f(x)-\phi(x)}_\infty\\
    &\leq\epsilon+\delta,
\end{align*}
for any $\sigma\in\Sigma$. Therefore, 
\begin{align*}
    \norm{f(x)-\frac{1}{\abs{\Sigma}}\sum_{i=1}^{\abs{\Sigma}}\phi(\sigma_i^{-1}x)}_\infty 
    &= \norm{\frac{1}{\abs{\Sigma}}\sum_{i=1}^{\abs{\Sigma}}f(x)-\frac{1}{\abs{\Sigma}}\sum_{i=1}^{\abs{\Sigma}}\phi(\sigma_i^{-1}x)}_\infty\\
    &\leq \frac{1}{\abs{\Sigma}}\sum_{i=1}^{\abs{\Sigma}}\norm{f(x)-\phi(\sigma_i^{-1}x)}_\infty\\
    &\leq\epsilon+\delta.
\end{align*}
Note that $\frac{1}{\abs{\Sigma}}\sum_{i=1}^{\abs{\Sigma}}\phi(\sigma_i^{-1}x)\in\DNNS$ and $\delta$ can be arbitrarily small, hence we have
\begin{align*}
    \inf_{f_\omega\in\DNNS}\norm{f(x)-f_\omega(x)}_\infty \leq \epsilon.
\end{align*}
Since $f$ can be any function in $\Gamma_\Sigma$, we have
\begin{equation*}
    \sup_{f\in\Gamma_\Sigma}\inf_{f_\omega\in\DNNS}\norm{f-f_\omega}_\infty\leq\epsilon.
\end{equation*}
This completes the proof of \ref{lemma:discriminator}.
\end{proof}
Since the weight of the $\Sigma$-symmetrization layer $W_\Sigma$ is determined by $\Sigma$ and the ReLU function is 1-Lipschitz, combined with the uniform bound for the weights by $K_1$, the Lipschitz constant of each layer has a uniform bound. Therefore, we have $\sup_{f_\omega\in\DNNS}\norm{f}_{\text{Lip}}\leq \tilde{H}$ for some $\tilde{H}>0$ that depends on $\Sigma,W_1,L_1$ and $K_1$, where $\norm{f}_{\text{Lip}}$ stands for the Lipschitz constant of the function $f$. This observation is useful to prove the bound for $\Delta_2$ in the following subsection.

\begin{proof}[Proof of \Cref{prop:nonequivariance}]
    We have 
    \begin{align*}
        \int_\CX \abs{f-f_\omega}^2\diff{x} \leq \text{vol}(\CX)\norm{f-f_\omega}_\infty^2.
    \end{align*}
    For the left-hand side, we claim that
    \begin{align*}
        \int_\CX \abs{f-f_\omega}^2\diff{x} = \int_\CX \abs{f-S_\Sigma[f_\omega]}^2\diff{x} + \int_\CX \abs{f_\omega-S_\Sigma[f_\omega]}^2\diff{x}.
    \end{align*}
    To prove the above equality, it is sufficient to show that
    \begin{equation}
        \int_\CX ff_\omega\diff{x} = \int_\CX \left(S_\Sigma[f_\omega]f-(S_\Sigma[f_\omega])^2+S_\Sigma[f_\omega]f_\omega\right)\diff{x}
    \end{equation}
    We show $\int(S_\Sigma[f_\omega])^2\diff{x} = \int  S_\Sigma[f_\omega]f_\omega\diff{x}$ and $\int ff_\omega\diff{x} = \int  S_\Sigma[f_\omega]f\diff{x}$. First, we have
    \begin{align*}
        \int_\CX(S_\Sigma[f_\omega])^2\diff{x} 
        &= \int_\CX \frac{1}{\abs{\Sigma}}\frac{1}{\abs{\Sigma}}\sum_{i=1}^{\abs{\Sigma}}\sum_{j=1}^{\abs{\Sigma}}f_\omega(\sigma_i^{-1}x)f_\omega(\sigma_j^{-1}x)\diff{x}\\
        &= \frac{1}{\abs{\Sigma}}\sum_{i=1}^{\abs{\Sigma}} \int_\CX\frac{1}{\abs{\Sigma}}\sum_{j=1}^{\abs{\Sigma}}f_\omega(x)f_\omega(\sigma_j^{-1}\sigma_i x)\diff{x}\\
        &= \frac{1}{\abs{\Sigma}}\sum_{i=1}^{\abs{\Sigma}} \int_\CX\frac{1}{\abs{\Sigma}}\sum_{k=1}^{\abs{\Sigma}}f_\omega(x)f_\omega(\sigma_k^{-1} x)\diff{x}\\
        &= \frac{1}{\abs{\Sigma}}\sum_{i=1}^{\abs{\Sigma}} \int_\CX f_\omega(x)S_\Sigma[f_\omega(x)]\diff{x}\\
        &= \int_\CX S_\Sigma[f_\omega(x)]f_\omega(x)\diff{x}.
    \end{align*}
    On the other hand, we have
    \begin{align*}
        \int_\CX S_\Sigma[f_\omega(x)]f \diff{x}
        &= \frac{1}{\abs{\Sigma}}\sum_{i=1}^{\abs{\Sigma}} \int_\CX f_\omega(\sigma_i^{-1}x)f(x)\diff{x}\\
        &= \frac{1}{\abs{\Sigma}}\sum_{i=1}^{\abs{\Sigma}} \int_\CX f_\omega(x)f(\sigma_i x)\diff{x}\\
        &= \frac{1}{\abs{\Sigma}}\sum_{i=1}^{\abs{\Sigma}} \int_\CX f_\omega(x)f(x)\diff{x}\\
        &= \int_\CX f_\omega(x)f(x)\diff{x},
    \end{align*}
    where the second to the last equality is because $f$ is $\Sigma$-invariant.
\end{proof}

\subsection{Bound for $\Delta_2$}\label{proof:Delta2}
Before we bound the generator approximation error, we cite a few useful results. We denote by $\mathcal{S}^d(z_0,\dots,z_{N+1})$
the set of all continuous piece-wise linear functions $f:\mathbb{R}\to\mathbb{R}^d$ which have breakpoints at $z_0<z_1<\cdots<z_{N+1}$ and are constant on $(-\infty,z_0)$ and $(z_{N+1},\infty)$. 
The following lemma from \cite{yang2022capacity} extends the one-dimensional approximation result in \cite{daubechies2022nonlinear} to higher dimensions.

\begin{lemma}[Lemma 3.1 in \cite{yang2022capacity}]\label{lemma:piecelinear}
    Suppose $W\geq7d+1, L\geq 2$ and $n\leq (W-d-1)\floor{\frac{W-d-1}{6d}}\floor{\frac{L}{2}}$. Then for any $z_0<z_1<\cdots<z_{n+1}$, we have $\mathcal{S}^d(z_0,\dots,z_{n+1})\subset\mathcal{NN}(W,L)$.
\end{lemma}
We next provide the proof of the generator approximation error bound (\cref{lemma:generator}, restated below). The construction in the first step of the proof follows that of Lemma 3.2 in \cite{yang2022capacity}, and we add the group symmetrization in the second step of the proof.
\begin{proof}[Proof of \cref{lemma:generator}]
    First, by Lemma \ref{thm:sp-gan-main-result}, we have \[d_{\DNNS}(S^\Sigma[g_\sharp\rho],\widehat{\mu}_n) = d_{\DNNS}(S^\Sigma[g_\sharp\rho],S^{\Sigma}[\widehat{\mu}_n])\] due to $S_{\Sigma}[\DNNS] = \DNNS$.

    \textbf{Step 1:} Without loss of generality, we assume that $m:= n-1\geq 1$ and $\widehat{\mu}_n = \frac{1}{n}\sum_{i=0}^m \delta_{x_i}$. Let $\epsilon$ be sufficiently small such that $0<\epsilon<\frac{m}{n}\norm{x_i-x_{i-1}}_2$ for all $i=1,\dots,m$. By the absolute continuity of $\rho$, we can choose $2m$ points on $\mathbb{R}$: $z_{1/2}<z_1<z_{3/2}<\cdots<z_{m-1/2}<z_m$, such that
    \begin{align*}
    &\rho((-\infty,z_{1/2})) = \frac{1}{n}, \\
    &\rho((z_{i-1/2},z_{i})) = \frac{\epsilon}{m\norm{x_i-x_{i-1}}_2}, &&1\le i\le m,\\
    &\rho((z_i,z_{i+1/2})) = \frac{1}{n} -\frac{\epsilon}{m\norm{x_i-x_{i-1}}_2}, &&1\le i\le m-1,\\
    &\rho((z_m,\infty)) = \frac{1}{n} -\frac{\epsilon}{m\norm{x_m-x_{m-1}}_2}.
    \end{align*}
    We define the continuous piece-wise linear function $\phi :\mathbb{R} \to \mathbb{R}^d$ by
    \[
    \phi(z) := 
    \begin{cases}
    x_0  &z \in (-\infty, z_{1/2}), \\
    \frac{z_i-z}{z_i-z_{i-1/2}}x_{i-1} + \frac{z-z_{i-1/2}}{z_i-z_{i-1/2}}x_{i} &z\in [z_{i-1/2},z_i), \\
    x_i  & z\in [z_{i},z_{i+1/2}), \\
    x_m  &z\in [z_m,\infty).
    \end{cases}
    \]
    Since $\phi \in \mathcal{S}^d(z_{1/2},\dots, z_m)$ has $2m= 2n-2\leq (W_2-d-1) \lfloor \frac{W_2-d-1}{6d} \rfloor \lfloor \frac{L_2}{2} \rfloor+2$ breakpoints, by Lemma \ref{lemma:piecelinear}, $\phi \in \mathcal{NN}(W_2,L_2)$. We denote the line segment joining $x_{i-1}$ and $x_i$ by $\mathcal{L}_i :=\{(1-t)x_{i-1} + tx_i\in \mathbb{R}^d: 0<t\leq 1\}$.

    \textbf{Step 2:} Note that $\sigma\mathcal{L}_i$ is the line segment joining $\sigma x_{i-1}$ and $\sigma x_i$ for any $\sigma\in\Sigma$ since the group action is linear. Then $S^\Sigma[\phi_\sharp\rho]$ is supported on $\bigcup_{j=1}^{\abs{\Sigma}}\left(\cup_{i=1}^m \sigma_j\mathcal{L}_i \cup \{\sigma_j x_0\}\right)$ and $S^\Sigma[\phi_\sharp\rho](\{\sigma_j x_0\}) = \frac{1}{n\abs{\Sigma}}$, $S^\Sigma[\phi_\sharp\rho](\{\sigma_j x_i\}) =\frac{1}{n\abs{\Sigma}}-\frac{\epsilon}{m\abs{\Sigma}\norm{x_i-x_{i-1}}_2}$, $S^\Sigma[\phi_\sharp\rho](\sigma_j\mathcal{L}_i) =\frac{1}{n\abs{\Sigma}}$ for $i=1,\dots,m$ and $j=1,\dots,\abs{\Sigma}$. We define the sum of product measures
    \[
    \gamma = \sum_{j=1}^{\abs{\Sigma}}\ S^\Sigma[\phi_\sharp\rho]|_{\{\sigma_j x_0\}} \times \delta_{\sigma_j x_0} + \sum_{j=1}^{\abs{\Sigma}}\sum_{i=1}^{m}   S^\Sigma[\phi_\sharp\rho]|_{\sigma_j \mathcal{L}_i}\times \delta_{\sigma_j x_i}.
    \]
    It is easy to verify $\gamma$ is a coupling of $S^\Sigma[\phi_\sharp\rho]$ and $S^\Sigma[\widehat{\mu}_n]$. Thus we have
    \begin{align*}
    \mathcal{W}_1(S^\Sigma[g_\sharp\rho],S^{\Sigma}[\widehat{\mu}_n]) 
        &\leq \int_{\mathbb{R}^d\times\mathbb{R}^d} \norm{x-y}_2 \diff\gamma(x,y)\\
        & =\sum_{j=1}^{\abs{\Sigma}}\sum_{i=1}^m\int_{\sigma_j\mathcal{L}_i\backslash\{\sigma_j x_i\}}\norm{\sigma_j x_i-y}_2\diff S^\Sigma[\phi_\sharp\rho](y)\\
        &\leq \sum_{j=1}^{\abs{\Sigma}}\sum_{i=1}^m\norm{\sigma_j x_i-\sigma_j x_{i-1}}_2 S^\Sigma[\phi_\sharp\rho](\sigma_j\mathcal{L}_i\backslash\{\sigma_j x_i\})\\
        &\leq \sum_{j=1}^{\abs{\Sigma}}\sum_{i=1}^m\norm{x_i-x_{i-1}}_2 \frac{\epsilon}{m\abs{\Sigma}\norm{x_i-x_{i-1}}_2}\\
        &=\epsilon,
    \end{align*}
    where the last inequality is due to $\Sigma$-actions being 1-Lipschitz. Since functions in $\DNNS$ are $\tilde{H}$-Lipschitz, we have
    \[
    d_{\DNNS}(S^\Sigma[g_\sharp\rho],S^{\Sigma}[\widehat{\mu}_n]) \leq \tilde{H}\cdot\mathcal{W}_1(S^\Sigma[g_\sharp\rho],S^{\Sigma}[\widehat{\mu}_n]) \leq\tilde{H}\epsilon.
    \]
    Since $\epsilon$ can be arbitrarily small, we have $d_{\DNNS}(S^\Sigma[g_\sharp\rho],S^{\Sigma}[\widehat{\mu}_n])=0$.
\end{proof}

\subsection{Bound for $\Delta_3$}\label{proof:Delta3}
We first cite some useful lemmas before establishing the bound for $\Delta_3$. First, we need the following bound of the Rademacher complexity by the Dudley's entropy integral.
\begin{lemma}\label{lemma:Dudley}
    Suppose $\mathcal{F}$ is a family of functions mapping $\mathcal{X}$ to $[-M,M]$ for some $M>0$. Also assume that $0\in\mathcal{F}$ and $\mathcal{F} = -\mathcal{F}$. Let $\xi=\{\xi_1,\dots,\xi_n\}$ be a set of independent random variables that take values on $\{-1,1\}$ with equal probabilities, $i = 1,\dots,n$. $x_1,x_2,\dots,x_n\in\mathcal{X}$. Then we have 
    \begin{equation*}
    \E_{\xi}\sup_{f\in\mathcal{F}}\abs{\frac{1}{n}\sum_{i=1}^n\xi_if(x_i)}\leq\inf_{\alpha>0} 4\alpha+\frac{12}{\sqrt{n}}\int_{\alpha}^{M}\sqrt{\log\mathcal{N(\mathcal{F},\epsilon,\norm{\cdot}_{\infty})}}\diff{\epsilon},
    \end{equation*}
    therefore, if $\mu$ is a distribution on $\CX$ and $\widehat{\mu}_n$ is its empirical distribution, we have
    \begin{equation*}
        \E[d_\mathcal{F}(\widehat{\mu}_n,\mu)]\leq\inf_{\alpha>0} 8\alpha+\frac{24}{\sqrt{n}}\int_{\alpha}^{M}\sqrt{\log\mathcal{N(\mathcal{F},\epsilon,\norm{\cdot}_{\infty})}}\diff{\epsilon}.
    \end{equation*}
\end{lemma}
The proof of \cref{lemma:Dudley} is standard using the dyadic path., e.g. see the proof of Lemma A.5. in \cite{bartlett2017spectrally}, which is modified in Lemma 3 in \cite{chen2023sample}.

The following lemma is a direct consequence of Lemma 6 in \cite{gottlieb2017efficient}.
\begin{lemma}\label{lemma:coverbycover}
    Let $\mathcal{F}$ be the family of $H$-Lipschitz functions mapping the metric space $(\mathcal{X},\norm{\cdot}_2)$ to $[-M,M]$ for some $M>0$. Then we have 
    \begin{equation*}
        \mathcal{N}(\mathcal{F},\epsilon,\,\norm{\cdot}_{\infty})\leq(\frac{c_1M}{\epsilon})^{\mathcal{N}(\mathcal{X},\frac{c_2\epsilon}{H})},
    \end{equation*}
    where $c_1\geq1$ and $c_2\leq1$ are some absolute constants independent of $\mathcal{X}$, $M$, and $\epsilon$.
\end{lemma}

\begin{lemma}[Theorem 3 in \cite{sokolic2017generalization}]\label{lemma:ratioofcovering}
Assume that $\mathcal{X}=\Sigma\times\mathcal{X}_0$. If for some $\epsilon>0$ we have\\
1) $\norm{\sigma(x)-\sigma'(x')}_2>2\epsilon$, $\forall x,x'\in\mathcal{X}_0, \sigma\neq\sigma'\in\Sigma$; and\\
2) $\norm{\sigma(x)-\sigma(x')}_2\geq\norm{x-x'}_2$, $\forall x,x'\in\mathcal{X}_0, \sigma\in\Sigma$,\\
then we have
\[
\frac{\mathcal{N}(\mathcal{X}_0,\epsilon)}{\mathcal{N}(\mathcal{X},\epsilon)}\leq\frac{1}{\abs{\Sigma}}.
\]
\end{lemma}

\begin{lemma}[Lemma 6 in \cite{chen2023sample}]\label{lemma:scalingofcoveringnumber}
    Let $\mathcal{X}$ be a subset of $\mathbb{R}^d$ and $\bar{\epsilon}>0$. Then there exists a constant $C_{d,\bar{\epsilon}}$ that depends on $d$ and $\bar{\epsilon}$ such that for $\epsilon\in(0,1)$ we have
    \begin{align*}
        \mathcal{N}(\mathcal{X},\epsilon)\leq C_{d,\bar{\epsilon}}\cdot\frac{\mathcal{N}(\mathcal{X},\bar{\epsilon})}{\epsilon^d}.
    \end{align*}
\end{lemma}

Following the statistical analysis in \cite{chen2023sample}, we can bound the statistical error $d_{\Gamma_\Sigma}(\widehat{\mu}_n,\mu)$ as follows.

\begin{proof}[Proof of \cref{lemma:statistical-target}] By the definition of the probability metric $d_\Gamma$ \eqref{eq:IPM}, we have
    \begin{align}\label{eq:IPMsup}
    d_{\Gamma_\Sigma}(\widehat{\mu}_n,\mu)&=\abs{d_{\Gamma_\Sigma}(\widehat{\mu}_n,\mu)}\nonumber\\
    &= \abs{ \sup_{\gamma\in\Gamma_{\Sigma}}\left\{\frac{1}{n}\sum_{i=1}^n\gamma(x_i)-\E_{\mu}[\gamma]\right\}}\nonumber\\
    &\leq\sup_{\gamma\in\Gamma_{\Sigma}}\abs{\E_{\mu}[\gamma]-\frac{1}{n}\sum_{i=1}^n\gamma(x_i)}\nonumber\\
    &=\sup_{\gamma\in\Gamma_{\Sigma}}\abs{\E_\mu[\gamma]-\frac{1}{n}\sum_{i=1}^n\gamma\left(T_0(x_i)\right)}\nonumber\\
    &\stackrel{\text{$(a)$}}{\leq} \sup_{\gamma\in\text{Lip}_H(\mathcal{X}_0)}\abs{\E_{\mu_{\mathcal{X}_0}}[\gamma]-\frac{1}{n}\sum_{i=1}^n\gamma\left(T_0(x_i)\right)},
\end{align}
where inequality $(a)$ is due to the fact that $\E_{\mu}[\gamma] = \E_{\mu_{\mathcal{X}_0}}[\gamma|_{\mathcal{X}_0}]$ since $\mu$ is $\Sigma$-invariant and $\gamma\in\Gamma_{\Sigma}$, and the fact that if $\gamma\in\Gamma_{\Sigma}$, then $\gamma|_{\mathcal{X}_0}\in\text{Lip}_H(\mathcal{X}_0)$, where $\gamma|_{\mathcal{X}_0}$ is the restriction of $\gamma$ on $\mathcal{X}_0$.

Denote by $X' = \{x'_1,x'_2,\dots,x'_n\}$ i.i.d. samples drawn from $\mu_{\mathcal{X}_0}$. Also note that \[T_0(x_1),\dots,T_0(x_n)\] can be viewed as i.i.d. samples on $\mathcal{X}_0$ drawn from $\mu_{\mathcal{X}_0}$. Therefore the expectation
\begin{align*}
\E_{X}\sup_{\gamma\in\text{Lip}_H(\mathcal{X}_0)}\abs{E_{\mu_{\mathcal{X}_0}}[\gamma]-\frac{1}{n}\sum_{i=1}^n\gamma(T_0(x_i))}
\end{align*}
can be replaced by the equivalent quantity
\[
\E_{X}\sup_{\gamma\in\text{Lip}_H(\mathcal{X}_0)}\abs{E_{\mu_{\mathcal{X}_0}}[\gamma]-\frac{1}{n}\sum_{i=1}^n\gamma(x_i)},
\]
where $X = \{x_1,x_2,\dots,x_m\}$ are i.i.d. samples on $\mathcal{X}_0$ drawn from $\mu_{\mathcal{X}_0}$. Then we have
\begin{align*}
\E_{X}\sup_{\gamma\in\text{Lip}_H(\mathcal{X}_0)}\abs{E_{\mu_{\mathcal{X}_0}}[\gamma]-\frac{1}{n}\sum_{i=1}^n\gamma(x_i)}
&= \E_{X}\sup_{\gamma\in\text{Lip}_H(\mathcal{X}_0)}\abs{E_{X'}\left(\frac{1}{n}\sum_{i=1}^n\gamma(x'_i)\right)-\frac{1}{n}\sum_{i=1}^n\gamma(x_i)}\\
&\leq \E_{X,X'}\sup_{\gamma\in\text{Lip}_H(\mathcal{X}_0)}\abs{\frac{1}{n}\sum_{i=1}^n\gamma(x'_i)-\frac{1}{n}\sum_{i=1}^n\gamma(x_i)}\\
&=\E_{X,X',\xi}\sup_{\gamma\in\text{Lip}_H(\mathcal{X}_0)}\abs{\frac{1}{n}\sum_{i=1}^n\xi_i\left(\gamma(x'_i)-\gamma(x_i)\right)}\\
&\leq 2\E_{X,\xi}\sup_{\gamma\in\text{Lip}_H(\mathcal{X}_0)}\abs{\frac{1}{n}\sum_{i=1}^n\xi_i\gamma(x'_i)}\\
&= \inf_{\alpha>0} 8\alpha+\frac{24}{\sqrt{n}}\int_{\alpha}^{M}\sqrt{\log\mathcal{N}(\mathcal{F}_0,\epsilon,\,\norm{\cdot}_{\infty})}\diff{\epsilon},
\end{align*}
where $\mathcal{F}_0 = \{\gamma\in\text{Lip}_H(\mathcal{X}_0):\norm{\gamma}_\infty\leq M\}$. 

For $d\geq2$, from Lemma \ref{lemma:coverbycover}, we have $\log\mathcal{N}(\mathcal{F}_0,\epsilon,\,\norm{\cdot}_{\infty})\leq\mathcal{N}(\mathcal{X}_0,\frac{c_2\epsilon}{H})\log(\frac{c_1M}{\epsilon})$. We fix an $\bar{\epsilon}>0$ such that $\mathcal{N}(\mathcal{X},\frac{c_2\bar{\epsilon}}{H})=1$, and select $\epsilon^*$ such that $\frac{c_2\epsilon^*}{H}\leq1$ and $\frac{c_2\epsilon^*}{H}\leq\epsilon_\Sigma$; that is, $\epsilon^*\leq\min\left(\frac{H}{c_2},\frac{H\epsilon_\Sigma}{c_2}\right)$, so that by Assumption~\ref{assumption:new}, Lemma~\ref{lemma:ratioofcovering} and Lemma~\ref{lemma:scalingofcoveringnumber}, we have 
\begin{align*}
    \mathcal{N}(\mathcal{X}_0,\frac{c_2\epsilon}{H})\log(\frac{c_1M}{\epsilon})
    &\leq \CN(\CX_0\backslash A_0(\frac{c_2\epsilon}{H}),\frac{c_2\epsilon}{H})\log(\frac{c_1M}{\epsilon}) + \CN(A_0(\frac{c_2\epsilon}{H}),\frac{c_2\epsilon}{H})\log(\frac{c_1M}{\epsilon})\\
    &\leq \frac{\mathcal{N}(\mathcal{X},\frac{c_2\epsilon}{H})}{\abs{\Sigma}}\log(\frac{c_1M}{\epsilon}) + \bar{c}\epsilon^r\CN(\CX,\frac{c_2\epsilon}{H})\log(\frac{c_1M}{\epsilon})\\
    &\leq \frac{C_{d,\bar{\epsilon}}H^d}{\abs{\Sigma}c_2^d\epsilon^d}\log(\frac{c_1M}{\epsilon}) + \frac{\bar{c}C_{d,\bar{\epsilon}}H^d}{c_2^d\epsilon^{d-r}}\log(\frac{c_1M}{\epsilon})
\end{align*}
when $\epsilon<\epsilon^*$, where $\bar{c}>0$ is some constant implied by Assumption~\ref{assumption:new}. Therefore, for sufficiently small $\alpha$, we have 
\begin{align}\label{eq:dudleybound}
&\int_{\alpha}^{M}\sqrt{\log\mathcal{N}(\mathcal{F}_0,\epsilon,\,\norm{\cdot}_{\infty})}\diff{\epsilon}\nonumber\\
=& \int_{\alpha}^{\epsilon^*}\sqrt{\log\mathcal{N}(\mathcal{F}_0,\epsilon,\,\norm{\cdot}_{\infty})}\diff{\epsilon} + \int_{\epsilon^*}^{M}\sqrt{\log\mathcal{N}(\mathcal{F}_0,\epsilon,\,\norm{\cdot}_{\infty})}\diff{\epsilon}\nonumber\\
\leq & \int_{\alpha}^{\epsilon^*}\sqrt{\frac{C_{d,\bar{\epsilon}}H^d}{\abs{\Sigma}c_2^d\epsilon^d}\log(\frac{c_1M}{\epsilon})}\diff{\epsilon} + \int_{\alpha}^{\epsilon^*}\sqrt{\frac{\bar{c}C_{d,\bar{\epsilon}}H^d}{\abs{\Sigma}c_2^d\epsilon^{d-r}}\log(\frac{c_1M}{\epsilon})}\diff{\epsilon} + \int_{\epsilon^*}^{M}\sqrt{\log\mathcal{N}(\mathcal{F}_0,\epsilon,\,\norm{\cdot}_{\infty})}\diff{\epsilon}.
\end{align}
For any $s>0$, we can choose $\epsilon^*$ to be sufficiently small, such that $\log(\frac{c_1M}{\epsilon})\leq\frac{1}{\epsilon^s}$ when $\epsilon\leq\epsilon^*$. Therefore, if we let $D_{\mathcal{X},H}=\sqrt{\frac{C_{d,\bar{\epsilon}}H^d}{c_2^d}}$, then for the first term in \eqref{eq:dudleybound}, we have 
\begin{align*}      \int_{\alpha}^{\epsilon^*}\sqrt{\frac{C_{d,\bar{\epsilon}}H^d}{\abs{\Sigma}c_2^d\epsilon^d}\log(\frac{c_1M}{\epsilon})}\diff{\epsilon}
&\leq D_{\mathcal{X},H}\int_{\alpha}^{\epsilon^*}\sqrt{\frac{1}{\abs{\Sigma}\epsilon^{d+s}}}\diff{\epsilon}\\
&\leq D_{\mathcal{X},H}\int_{\epsilon}^{\infty}\sqrt{\frac{1}{\abs{\Sigma}\epsilon^{d+s}}}\diff{\epsilon}\\
&=\frac{D_{\mathcal{X},H}}{\sqrt{\abs{\Sigma}}}\cdot\frac{\alpha^{1-\frac{d+s}{2}}}{\frac{d+s}{2}-1}.
\end{align*}
Notice that the third integral in \eqref{eq:dudleybound} is bounded while the first integral diverges as $\alpha$ tends to zero; the second integral is either bounded (if $d-r<2$ and $s$ is small) or is of order $\simeq \alpha^{1-\frac{d-r+s}{2}}$ (if $d-r\geq 2$), so we can optimize the majorizing terms
\[
8\alpha + \frac{24}{\sqrt{n}}\cdot\frac{D_{\mathcal{X},H}}{\sqrt{\abs{\Sigma}}}\cdot\frac{\alpha^{1-\frac{d+s}{2}}}{\frac{d+s}{2}-1}
\]
with respect to $\alpha$, to obtain \[
\alpha = \left(\frac{9}{n}\right)^{\frac{1}{d+s}}\cdot\left(\frac{D_{\mathcal{X},H}^2}{\abs{\Sigma}}\right)^{\frac{1}{d+s}},
\]
so that
\begin{align}\label{eq:hiddenconstant}
    &\inf_{\alpha>0} 8\alpha+\frac{24}{\sqrt{n}}\int_{\alpha}^{M}\sqrt{\log\mathcal{N}(\mathcal{F}_0,\epsilon,\,\norm{\cdot}_{\infty})}\diff{\epsilon}\nonumber\\
    &\leq 8\left(\frac{9}{n}\right)^{\frac{1}{d+s}}\cdot\left(\frac{D_{\mathcal{X},H}^2}{\abs{\Sigma}}\right)^{\frac{1}{d+s}} + \frac{24}{(\frac{d+s}{2}-1)}\left(\frac{9}{n}\right)^{\frac{1}{d+s}}\cdot\left(\frac{D_{\mathcal{X},H}^2}{\abs{\Sigma}}\right)^{\frac{1}{d+s}} + o\left(\frac{1}{n^{\frac{1}{d+s}}}\right).
\end{align}
Therefore, for sufficiently large $n$, we have 
\begin{align}\label{eq:fullconstant}
\E_{X}\sup_{\gamma\in\text{Lip}_H(\mathcal{X}_0)}\abs{E_{\mu_{\mathcal{X}_0}}[\gamma]-\frac{1}{n}\sum_{i=1}^n\gamma(x_i)}&\leq \left(8+\frac{24}{(\frac{d+s}{2}-1)}\right)\left(\frac{9D_{\mathcal{X},H}^2}{\abs{\Sigma}n}\right)^{\frac{1}{d+s}}
+ o\left(\frac{1}{n^{\frac{1}{d+s}}}\right)\nonumber\\
&:= C_{\CX,H,d,s}\left(\frac{1}{\abs{\Sigma}n}\right)^{\frac{1}{d+s}}+ o\left(\left(\frac{1}{n}\right)^{\frac{1}{d+s}}\right).
\end{align}

For $d=1$, the first integral in \eqref{eq:dudleybound} does not have a singularity at $\alpha = 0$ if $1+s<2$. On the other hand, replacing the interval $[0,1]$ by an interval of length $\text{diam}(\mathcal{X}_0)$ in Lemma 5.16 in \cite{van2014probability}, there exists a constant $c>0$ such that 
    \[
    \mathcal{N}(\mathcal{F}_0,\epsilon,\,\norm{\cdot}_{\infty})\leq e^{\frac{cH\cdot\text{diam}(\mathcal{X}_0)}{\epsilon}}\,\,\, \text{for}\,\, \epsilon<M.\]
Therefore, we have 
\begin{align*}
    8\alpha+\frac{24}{\sqrt{n}}\int_{\alpha}^{M}\sqrt{\log\mathcal{N}(\mathcal{F}_0,\epsilon,\,\norm{\cdot}_{\infty})}\diff{\epsilon} \leq 8\alpha + \frac{24}{\sqrt{n}}\int_{\alpha}^{M}\sqrt{\frac{cH\cdot\text{diam}(\mathcal{X}_0)}{\epsilon}}\diff{\epsilon},
\end{align*}
whose minimum is achieved at $\alpha = \frac{9cH\cdot\text{diam}(\mathcal{X}_0)}{n}$. This implies that
\begin{align*}
&\inf_{\alpha>0} 8\alpha+\frac{24}{\sqrt{n}}\int_{\alpha}^{M}\sqrt{\log\mathcal{N}(\mathcal{F}_0,\epsilon,\,\norm{\cdot}_{\infty})}\diff{\epsilon}\\
\leq & \frac{72cH\cdot\text{diam}(\mathcal{X}_0)}{n} + \frac{48H\sqrt{c}\cdot\text{diam}(\mathcal{X}_0)}{\sqrt{n}} - \frac{144cH\cdot\text{diam}(\mathcal{X}_0)}{n}\\
= & \frac{48H\sqrt{c}\cdot\text{diam}(\mathcal{X}_0)}{\sqrt{n}} - \frac{72cH\cdot\text{diam}(\mathcal{X}_0)}{n}.
\end{align*}
Hence, we have
\begin{align*}
\E_{X}\sup_{\gamma\in\text{Lip}_H(\mathcal{X}_0)}\abs{\E_{\mu_{\mathcal{X}_0}}[\gamma]-\frac{1}{n}\sum_{i=1}^n\gamma(x_i)}\leq \frac{48H\sqrt{c}\cdot\text{diam}(\mathcal{X}_0)}{\sqrt{n}} - \frac{72cH\cdot\text{diam}(\mathcal{X}_0)}{n}.
\end{align*}
This completes the proof.
\end{proof}

\subsection{Bound for $\Delta_4$}\label{proof:Delta4}
To prove the bound for $\E[d_{\DNNS\circ \GNN}(\rho,\widehat{\rho}_m)]$, we introduce the notion of pseudo-dimension from \cite{bartlett2019nearly}, which is another measure of complexity for a class of functions.
\begin{definition}[Pseudo-dimension]\label{definition:pseudo}
    Let $\CF$ be a class of functions that map $\CX$ to $\mathbb{R}$. The pseudo-dimension of $\CF$, denoted by $\text{Pdim}(\CF)$, is the largest integer $n$ for which there exists $(x_1,\dots,x_n,y_1,\dots,y_n)\in\CX^n\times\mathbb{R}^n$ such that for any $(b_1,\dots,b_n)\in\{0,1\}^n$, there exists $f\in\CF$ such that 
    \[
    \forall i: f(x_i)>y_i \quad\text{iff}\quad b_i=1.
    \]
\end{definition}
\begin{proof}[Proof of \cref{lemma:statistical-source}]
    First we show that $\sup_{f\in\DNNS\circ \GNN}\norm{f}_\infty$ is bounded. This is straightforward since we can add an additional clipping layer to the output of $\DNNS$ so that its output lies within, for example, $[-2M,2M]$. Note that such clipping does not impact the invariant discriminator approximation error, if we require $\epsilon\leq\frac{M}{2}$ in Lemma~\ref{lemma:discriminator}. By Corollary 35 in \cite{huang2022error}, we have 
    \[
    d_{\DNNS\circ \GNN}(\rho,\widehat{\rho}_m)\lesssim\sqrt{\frac{\text{Pdim}(\DNNS\circ \GNN)\log m}{m}}.
    \]
    By Theorem 7 in \cite{bartlett2019nearly}, we have $\text{Pdim}(\mathcal{NN}(W,L,N))\lesssim NL\log N$. Note that we can rewrite $\DNNS$ defined in \eqref{eq:invariant-discriminator} as a ReLU network with a larger width (increased by a factor of the group size) and the same depth $L_1$ as $\DNN$, despite the first layer where we multiply input $x$ with $W_{\sigma_i}$'s and the last averaging layer. Importantly, the number of free parameters $N_1$ remains exactly the same as $\DNN$. Hence we have
    \[
    \E[d_{\DNNS\circ \GNN}(\rho,\widehat{\rho}_m)]\lesssim \sqrt{\frac{(N_1+W_2^2L_2)(L_1+L_2)\log(N_1+W_2^2L_2)\log m}{m}},
    \]
    where we use the trivial bound $N_2\simeq W_2^2 L_2$. By Lemma~\ref{lemma:discriminator} and Lemma~\ref{lemma:generator}, we have $N_1\lesssim n\log n$ and $W_2^2 L_2\lesssim n$, so we have
    \begin{align*}
        \E[d_{\DNNS\circ \GNN}(\rho,\widehat{\rho}_m)]
        &\lesssim \sqrt{\frac{(n\log n+n)(\log n+n)\log(n\log n+n)\log m}{m}}\\
        &\lesssim \sqrt{\frac{n^2\log^2 n\log m}{m}}.
    \end{align*}
    Thus if $m\gtrsim n^{2+2/d}\log^3n$, we have $\E[d_{\DNNS\circ \GNN}(\rho,\widehat{\rho}_m)]\lesssim n^{-1/d}$.
\end{proof}

\subsection{Theorem~\ref{theorem:main} for $d=1$}

\begin{theorem}[Main Theorem ($d=1$)]\label{theorem:main1d}
    Let $\CX=\Sigma\times\CX_0$ be a subset of $\mathbb{R}$ and $\Gamma=\text{Lip}_H(\CX)$ and $\CX_0$ is an interval of finite length. Suppose the target distribution $\mu$ is $\Sigma$-invariant on $\CX$ and the noise source distribution $\rho$ is absolutely continuous on $\mathbb{R}$. Then there exists $\Sigma$-invariant discriminator architecture $\DNNS=S_\Sigma[\DNN]$, where $\DNN=\mathcal{NN}(W_1,L_1,N_1)$ as defined in \eqref{eq:invariant-discriminator} with $N_1\lesssim n\log n$ and $L_1\lesssim \log n$, and $\Sigma$-invariant generator architecture $\GNNS$, where $\GNN = \mathcal{NN}(W_2,L_2)$, with $W_2^2L_2\lesssim n$, such that if $m\gtrsim n^{4}\log^3n$, we have
    \begin{equation*}
        \E\left[d_{\Gamma}(S^\Sigma[(g_{n,m}^*)_\sharp\rho],\mu)\right] \lesssim \frac{\text{diam}(\CX_0)}{\sqrt{n}}.
    \end{equation*}
\end{theorem}
\begin{proof}
    It suffices to choose $\epsilon\simeq n^{-1}$ in Lemma~\ref{lemma:discriminator} with $L_1\lesssim \log n$ and $N_1\lesssim n\log n$, so that $\sup_{f\in\Gamma_\Sigma}\inf_{f_\omega\in\DNNS}\norm{f-f_\omega}_\infty\lesssim n^{-1}$. On the other hand, we can make \[\inf_{g\in\GNN}d_{\DNNS}(S^\Sigma[g_\sharp\rho],\widehat{\mu}_n)=0\] with $W_2^2L_2\lesssim n$ by Lemma~\ref{lemma:generator}. Finally, $\E[d_{\Gamma_\Sigma}(\widehat{\mu}_n,\mu)]\lesssim \frac{\text{diam}(\CX_0)}{\sqrt{n}}$ and $\E[d_{\DNNS\circ \GNN}(\rho,\widehat{\rho}_m)]\lesssim n^{-1}$ by Lemma~\ref{lemma:statistical-target} and Lemma~\ref{lemma:statistical-source} respectively.
\end{proof}

\section{Proof of \cref{theorem:lowdimensional} when $d=1$}\label{proof:lowdimensional}

\begin{theorem}[$1$-dimensional submanifolds]\label{theorem:lowdimensional1d}
    Let $\CX=\Sigma\times\CX_0\subset\mathbb{R}^d$ and $\Gamma=\text{Lip}_H(\CX)$. Suppose $\CX_0$ is diffeomorphic to some interval of finite length and the target distribution $\mu$ is $\Sigma$-invariant on $\CX$ and the noise source distribution $\rho$ is absolutely continuous on $\mathbb{R}$. Then there exists $\Sigma$-invariant discriminator architecture $\DNNS=S_\Sigma[\DNN]$, where $\DNN=\mathcal{NN}(W_1,L_1,N_1)$ as defined in \eqref{eq:invariant-discriminator} with $N_1\lesssim n\log n$ and $L_1\lesssim \log n$, and $\Sigma$-invariant generator architecture $\GNNS$, where $\GNN = \mathcal{NN}(W_2,L_2)$, with $W_2^2L_2\lesssim n$, such that if $m\gtrsim n^{4}\log^3n$, we have
    \begin{equation*}
        \E\left[d_{\Gamma}(S^\Sigma[(g_{n,m}^*)_\sharp\rho],\mu)\right] \lesssim \frac{\text{peri}(\CX_0)}{\sqrt{n}},
    \end{equation*}
    where $\text{peri}(\CX_0)$ denotes the perimeter of $\CX_0$ in $\mathbb{R}^d$.
\end{theorem}
\begin{proof}
    It suffices to show that $\Delta_3\lesssim \frac{\text{peri}(\CX_0)}{\sqrt{n}}$, which directly follows the proof of Theorem XIV in \cite{tikhomirov1993varepsilon}.
\end{proof}
%%%%%%%%%%%%%%%%%%%%%%%%%%%%%%%%%%%%%%%%%%%%%%%%%%%%%%%%%%%%%%%%%%%%%%%%%%%%%%%
%%%%%%%%%%%%%%%%%%%%%%%%%%%%%%%%%%%%%%%%%%%%%%%%%%%%%%%%%%%%%%%%%%%%%%%%%%%%%%%

\section{Acknowledgement}
Z. Chen, M. Katsoulakis, L. Rey-Bellet are partially funded by AFOSR grant FA9550-21-1-0354. M.K. and L.
R.-B. are partially funded by NSF DMS-2307115. M.K. is partially funded by NSF TRIPODS CISE-1934846. Z. Chen
and W. Zhu are partially supported by NSF under DMS-2052525, DMS-2140982, and DMS-2244976. We also thank the anonymous reviewers for their constructive feedback to improve this manuscript.

\bibliographystyle{plain}
\bibliography{mybibfile.bib}

\end{document}